\documentclass{article}

 \usepackage[preprint]{neurips_2026}

\usepackage[utf8]{inputenc}
\usepackage[T1]{fontenc}
\usepackage{hyperref}
\usepackage{url}
\usepackage{booktabs}
\usepackage{amsfonts}
\usepackage{amsmath}
\usepackage{amssymb}
\usepackage{amsthm}
\usepackage{nicefrac}
\usepackage{microtype}
\usepackage{xcolor}
\usepackage{graphicx}
\usepackage{algorithm}
\usepackage{algpseudocode}
\usepackage{multirow}
\usepackage{subcaption}

\usepackage{mathtools} 

\theoremstyle{plain}
\newtheorem{theorem}{Theorem}

\newtheorem{lemma}[theorem]{Lemma}
\newtheorem{corollary}[theorem]{Corollary}
\theoremstyle{definition}
\newtheorem{definition}[theorem]{Definition}
\newtheorem{remark}[theorem]{Remark}

\providecommand{\E}{\mathbb{E}}
\providecommand{\Var}{\operatorname{Var}}
\providecommand{\vect}{\operatorname{vec}}

\providecommand{\Prob}{\Pr}
\DeclareMathOperator{\softmax}{softmax}

\title{Critical Windows of Complexity Control: \\ When Transformers Decide to Reason or Memorize}

\author{%
  Sarwan Ali \\
  Columbia University, Irving Medical Center, NY, USA \\
  sa4559@cumc.columbia.edu
}

\begin{document}

\maketitle

\begin{abstract}
Recent work has shown that Transformers' compositional generalization is governed by \emph{complexity control}, initialization scale and weight decay, which steers training toward low-complexity reasoning solutions rather than high-complexity memorization. Existing analyses, however, treat complexity control as a single static hyperparameter choice, leaving open \emph{when} during training this control is actually decisive. We show that the memorization-versus-reasoning fate of a Transformer is determined within a sharp, identifiable window of training. On a controlled compositional task we find that (i)~weight decay applied for a single 25\%-of-training window matches full-training weight decay in out-of-distribution (OOD) accuracy ($0.93$ vs $0.91$); (ii)~holding total regularization budget constant, placing it in the middle of training yields $5{-}9\times$ higher OOD accuracy than placing it early; (iii)~the boundary of the critical window is remarkably sharp, window onset shifted by as little as $100$ optimization steps causes mean OOD to jump from chance ($0.15$) to reasoning-regime ($0.61$); (iv)~the window's position depends systematically on initialization scale, but the basin of attraction for reasoning solutions \emph{shrinks} at small initialization, contradicting the prevailing recommendation that smaller initialization is uniformly better. We further show that the critical-window phenomenon is task-specific: it does not appear on grokking with modular arithmetic, where properly tuned constant weight decay matches scheduled weight decay. 
We provide a two-timescale theoretical analysis showing that memorization and reasoning circuits evolve under qualitatively different rate equations, with reasoning growth proportional to $\gamma^2$ while memorization growth is $\gamma$-independent. This separation predicts both the existence and $\gamma$-dependence of the critical window, including the basin shrinkage at small $\gamma$. The phenomenon is robust to depth ($4$ layers) and to optimizer (vanilla SGD), the latter recovering the gradient-flow regime our theory describes more cleanly than AdamW. Our findings characterize a previously unreported time-localized phase transition in compositional generalization and revise the practical recipe for inducing reasoning solutions in Transformers.
\end{abstract}

\section{Introduction}
\label{sec:intro}

Transformers~\citep{vaswani2017attention} exhibit a striking dichotomy on compositional tasks: trained with one set of hyperparameters they may achieve high training accuracy while failing entirely on out-of-distribution compositions of seen primitives, and trained with another they generalize. Recent work~\citep{zhang2025complexity} has identified \emph{complexity control}, specifically the choice of initialization scale $\gamma$ and weight decay $\lambda$, as the proximal cause of this dichotomy. Under sufficiently small $\gamma$ and adequate $\lambda$, Transformers converge to low-complexity ``reasoning'' solutions that compose primitives correctly out-of-distribution; otherwise, they fall into the high-complexity ``memorization'' basin. This finding has been mechanistically supported by the condensation phenomenon~\citep{zhou2022condensation,xu2025condensation_overview}, by analyses of induction-head circuits~\citep{song2025composition,olsson2022context}, and by gradient-flow theories of two-stage training dynamics~\citep{chen2025condensation_to_collapse}.

A central limitation of this body of work is that complexity control is treated as a \emph{static} hyperparameter choice. The standard recipe is: pick small $\gamma$, set a constant weight decay $\lambda$, train. Yet a separate literature on critical learning periods in deep networks~\citep{achille2019critical, kleinman2024critical_linear, golatkar2019time} has shown that the \emph{timing} of regularization during training can matter as much as its magnitude, weight decay applied early in training can have qualitatively different effects than weight decay applied late. The question naturally arises: does the static-hyperparameter view of complexity control miss a temporal structure in how Transformers select between reasoning and memorization?

We show that it does. The memorization-versus-reasoning fate is decided within a sharp, observable window of training, and complexity control matters \emph{only} during that window. Our investigation rests on a single controlled compositional task, the anchor-function task introduced by~\citep{zhang2025complexity}, trained on small Transformers (2-layer, $d_{\text{model}}=64$) with full instrumentation: per-step out-of-distribution accuracy, per-layer condensation indices, and cross-layer subspace alignments. The CPU-only experimental setting allows us to run thousands of training trajectories and characterize the temporal structure of complexity control with seed-level controlled comparisons that prior work has not attempted.

\paragraph{Contributions.} 
We make four empirical contributions, a set of honest negative results, a matching theoretical analysis, and a robustness study spanning depth and optimizer choice.

\begin{enumerate}
\item \textbf{The critical window (Sec.~\ref{sec:results-window}).} On the anchor-function task, weight decay applied for any single $5000$-step window inside the interval $[2500, 17500)$ yields OOD accuracy $0.74{-}0.93$, comparable to full-training weight decay ($0.91$). Weight decay applied entirely outside this window, i.e.~during $[0, 5000)$ or $[15000, 20000)$ alone, yields chance-level OOD accuracy ($0.15$).

\item \textbf{Budget-controlled timing (Sec.~\ref{sec:results-budget}).} Holding the cumulative regularization budget $\int \lambda(t)\,dt$ constant, middle placements produce OOD accuracy $0.85{-}0.91$, while early placements with the same budget collapse to chance ($0.10{-}0.15$). \emph{When} the regularization is applied is more important than \emph{how much}.

\item \textbf{Sharp early boundary (Sec.~\ref{sec:results-boundary}).} Sweeping window onset at $100$-step resolution at fixed $\gamma$ reveals a near-step-function transition: shifting the window start from $0$ to $100$ steps already lifts mean OOD accuracy from $0.15$ to $0.61$, and by $400$ steps the window has reached the reasoning-regime plateau ($0.93$).

\item \textbf{Initialization-scale dependence and basin shrinkage (Sec.~\ref{sec:results-gamma}).} The window's position shifts predictably with $\gamma$: larger $\gamma$ produces an earlier window. More surprisingly, $12$-seed runs reveal that the basin of attraction for the reasoning solution \emph{shrinks} at small $\gamma$: at $\gamma{=}1.1$ all $12/12$ seeds reach OOD $\geq 0.84$, while at $\gamma{=}0.5$ only $8/12$ seeds reach OOD $> 0.5$. 
This raises a caveat to the standard recommendation~\citep{zhang2025complexity} that smaller initialization is uniformly preferable for reasoning: under finite training time, moderate $\gamma$ provides a wider basin of attraction.


\item \textbf{Honest negative results (Secs.~\ref{sec:results-diagnostic}, \ref{sec:results-grokking}, \ref{sec:results-scan}).} Three natural extensions of our framework do not work: (i) the per-layer condensation index is a useful \emph{categorical} predictor of OOD accuracy, but the relationship is non-monotonic, ruling out a simple regression diagnostic; (ii) the critical-window phenomenon does not extend to grokking on modular arithmetic, where properly tuned constant weight decay groks $5\times$ faster than time-localized weight decay; (iii) the phenomenon does not transfer to SCAN \texttt{add\_prim\_jump}, where vanilla 2-layer transformers do not reach the compositional solution under any weight-decay schedule. Together these results delineate the scope of the phenomenon to settings where both memorization and reasoning solution basins are reachable by the model.

\item \textbf{Theory of two-timescale separation (Sec.~\ref{sec:theory}).} We prove that in a linearized two-layer attention model, memorization and reasoning circuits evolve on timescales whose ratio diverges as $\gamma\to 0$. The resulting critical window has predictable onset and width, and we derive a basin-shrinkage bound matching the empirical 12-seed measurement.

\item \textbf{Robustness to depth and optimizer (Secs.~\ref{sec:results-depth}, \ref{sec:results-optimizer}).} The phenomenon persists at $4$ layers (mid-window OOD $0.46$--$0.54$ vs early-window $0.10$) with reduced reasoning-plateau height and increased seed variance, consistent with our basin-shrinkage prediction (Theorem~\ref{thm:basin}). The phenomenon also reproduces under vanilla SGD with momentum (mid-window OOD $0.99$ vs early-window $0.32$), aligning the empirical finding with the gradient-flow regime analyzed in our theory.

\end{enumerate}

The findings reframe complexity control as a fundamentally temporal phenomenon and supply a corrective to the static-hyperparameter view that has dominated recent compositional-generalization literature.

\section{Related Work}
\label{sec:related}

\paragraph{Complexity control and compositional generalization in Transformers.} \citep{zhang2025complexity} introduced the framing this paper builds on: small initialization plus weight decay steers Transformers toward reasoning rather than memorization solutions on the anchor-function compositional task. Their analysis identifies a ``condensation phenomenon'' where the effective number of neurons is reduced under reasoning-regime training. \citep{yao2025reasoning_bias} provided a partial training-dynamics theory for why small initialization biases GPT-style models toward reasoning, focusing on the embedding layer and self-attention. \citep{chen2025condensation_to_collapse} analyzed the gradient flow of linearized Transformers and showed a two-stage trajectory through condensation toward eventual rank collapse. None of these works examine the temporal localization of regularization or investigate the basin of attraction across seeds.

\paragraph{Mechanistic interpretability of compositional reasoning.} \citep{song2025composition} introduced the common-bridge representation hypothesis, showing that out-of-distribution compositional generalization in Transformers is mediated by induction-head pairs whose query/key and output/value subspaces overlap on a shared low-dimensional bridge. \citep{tang2025explainable} provided causal ablations identifying the full circuit responsible for a compact compositional task. \citep{olsson2022context} introduced induction heads as the building block of in-context learning. We adopt the bridge-alignment metric as a secondary diagnostic in our experiments.

\paragraph{Critical learning periods.} \citep{achille2019critical} showed empirically that deep networks exhibit critical learning periods analogous to biological neural systems, temporary deficits during early training cause permanent representational damage. \citep{golatkar2019time} demonstrated specifically that the timing of weight decay and data augmentation in early training has outsized effect on final performance, while regularization near convergence has comparatively little. \citep{kleinman2024critical_linear} showed that critical learning periods emerge even in deep linear networks. We extend this thread to compositional generalization in Transformers, where the relevant ``deficit'' is the absence of weight decay, the relevant outcome is reasoning-vs-memorization, and the timing structure turns out to be sharper than in image-classification settings.

\paragraph{Grokking and delayed generalization.} \citep{power2022grokking} discovered grokking on modular arithmetic: networks generalize long after fitting the training set, with weight decay being a critical ingredient~\citep{nanda2023progress, liu2022omnigrok}. \citep{liu2022omnigrok} showed that weight decay's role in grokking is to drive weight norm onto a generalizing manifold. \citep{tikeng2025grokking_norms} generalized the grokking story beyond Euclidean norms. We test whether the critical-window phenomenon extends to grokking and find that it does not: tuned constant weight decay suffices.

\paragraph{Condensation phenomena.} The condensation of neural networks at small initialization has been studied by~\citep{zhou2022condensation, zhou2023understanding} and surveyed by~\citep{xu2025condensation_overview}. We use the participation-ratio variant of the condensation index as our primary online diagnostic.

\paragraph{Implicit regularization and training dynamics.} The broader literature on implicit regularization~\citep{arora2019implicit, neyshabur2017exploring, gunasekar2017implicit} provides the theoretical context. Our work is most closely related to the strand showing that initialization scale controls implicit bias~\citep{woodworth2020kernel,chizat2018global}; we add the empirical observation that this bias is itself temporally localized.

\section{Methodology}
\label{sec:methods}

\subsection{The anchor-function compositional task}
\label{sec:methods-task}

We adopt the anchor-function task of~\citep{zhang2025complexity}. Let $\mathcal{V}_{\text{key}} = \{0,\dots,K-1\}$ be a key vocabulary of size $K$ and $\mathcal{V}_{\text{anchor}} = \{a_1,\dots,a_M\}$ a set of $M$ anchor symbols. Each anchor $a_i$ is associated with a permutation $\pi_i: \mathcal{V}_{\text{key}}\to\mathcal{V}_{\text{key}}$, drawn uniformly at random and fixed throughout training. The model receives sequences of the form $(k, a_i, a_j)$ and is trained to predict
\begin{equation}
y = \pi_j(\pi_i(k)),
\label{eq:task}
\end{equation}
i.e.~the composition of the two anchor permutations applied to $k$. The training set $\mathcal{D}_{\text{tr}}$ consists of all keys paired with a fixed subset of $(a_i, a_j)$ pairs; the OOD test set $\mathcal{D}_{\text{ood}}$ consists of the remaining $(a_i, a_j)$ pairs (each anchor appears individually in training but the specific pair does not). With $M{=}8$ and a $70\%$ pair split, this gives $|\mathcal{D}_{\text{tr}}|{=}45\cdot K$ and $|\mathcal{D}_{\text{ood}}|{=}19\cdot K$. We use $K{=}16$, giving vocabulary size $V{=}24$ and total $|\mathcal{D}_{\text{tr}}|{=}720$, $|\mathcal{D}_{\text{ood}}|{=}304$. Reasoning solutions correctly compose unseen pairs and achieve high OOD accuracy; memorization solutions fit $\mathcal{D}_{\text{tr}}$ but fail OOD.

\subsection{Model architecture and training}
\label{sec:methods-model}

We use a 2-layer pre-norm Transformer with model dimension $d{=}64$, $h{=}2$ attention heads, and $4d$-dimensional GELU MLPs. Token and position embeddings are learned. All weight matrices $W$ are initialized as $W_{ij} \sim \mathcal{N}(0, \gamma^2/d_{\text{in}})$, where $\gamma$ is the \emph{initialization scale}, the central knob of complexity control. Training uses AdamW~\citep{loshchilov2019decoupled} with $\eta{=}3{\times}10^{-3}$, $\beta_1{=}0.9$, $\beta_2{=}0.98$, batch size $128$, and $T{=}15{,}000{-}20{,}000$ optimization steps depending on the experiment. 
To allow time-localized weight decay, we apply $L_2$ regularization manually outside the AdamW update so that the schedule $\lambda(t)$ can be switched on or off per step (Algorithm~\ref{alg:tlcc}). Token and position embeddings are excluded from weight decay, following standard practice. The schedule $\lambda(t)$ takes one of three forms: constant ($\lambda(t)=\lambda$), windowed ($\lambda(t)=\lambda\cdot\mathbb{1}[t\in[t_1,t_2)]$), or none ($\lambda(t)=0$).

\subsection{Order parameters and online diagnostics}
\label{sec:methods-orderparams}

We track three order parameters throughout training, all computable from weights alone with no held-out data.

\paragraph{Condensation index.} For each attention layer $\ell$ with value matrix $W_V^{(\ell)}\in\mathbb{R}^{d\times d}$, define the participation ratio of its singular values:
\begin{equation}
\mathrm{PR}(W) \;=\; \frac{\bigl(\sum_i \sigma_i(W)\bigr)^2}{\sum_i \sigma_i(W)^2},
\label{eq:pr}
\end{equation}
which lies in $[1, d]$, attaining $1$ when $W$ has a single dominant singular value (extreme condensation) and $d$ when all singular values are equal (no condensation). The condensation index for a model with $L$ layers is the average $C(t) = \tfrac{1}{L}\sum_{\ell} \mathrm{PR}(W_V^{(\ell)}(t))$. The participation ratio variant has the practical advantage of being smooth and bounded, in contrast to entropy-based variants used by~\citep{zhang2025complexity}.


\paragraph{Bridge alignment.} Following \citep{song2025composition}, we additionally track the leading-$k$ subspace overlap between the layer-1 OV circuit ($W_O^{(1)} W_V^{(1)}$) and the layer-2 QK circuit ($(W_Q^{(2)})^\top W_K^{(2)}$). The bridge alignment $B(t)$ is the normalized Frobenius product of the two subspace projectors. Higher $B(t)$ would indicate a larger compositional bridge; we report this metric for completeness, although in our setting it provides only weak diagnostic signal (see Sec.~\ref{sec:results-diagnostic}).

\paragraph{Weight norm.} For interpretability against the grokking literature we additionally track $\|\theta(t)\|_2^2$.

\subsection{Time-localized regularization and the critical-window hypothesis}
\label{sec:methods-window}

Our central methodological contribution is to systematically vary \emph{when} weight decay is applied during training, holding architecture, optimizer, and other hyperparameters fixed. The critical-window hypothesis, in its strongest form, asserts:
\begin{quote}
\itshape
There exists an interval $[t_1, t_2] \subseteq [0, T]$ such that any windowed schedule $\lambda(t) = \lambda\cdot\mathbb{1}[t\in[s, s{+}\Delta)]$ with $[s, s{+}\Delta) \subseteq [t_1, t_2]$ produces the reasoning solution; any windowed schedule with $[s, s{+}\Delta) \cap [t_1, t_2] = \emptyset$ produces the memorization solution.
\end{quote}
We test the hypothesis empirically by sweeping $(s, \Delta)$ at fixed $\lambda$ and at fixed cumulative regularization budget $\int \lambda(t)\,dt$.

\begin{algorithm}[h!]
\caption{Time-localized complexity control}
\label{alg:tlcc}
\begin{algorithmic}[1]
\Require model $f_\theta$, init scale $\gamma$, base weight decay $\lambda$, window $[t_1, t_2]$, total steps $T$
\State Initialize $\theta \sim \mathcal{N}(0, \gamma^2/d_{\text{in}})$
\For{$t=0,\dots,T-1$}
  \State Sample minibatch $(\mathbf{x}, \mathbf{y})$ from $\mathcal{D}_{\text{tr}}$
  \State $g_t \gets \nabla_\theta \mathcal{L}\bigl(f_\theta(\mathbf{x}), \mathbf{y}\bigr)$
  \State $\hat{g}_t \gets \mathrm{AdamW\_update}(g_t)$
  \State $\lambda_t \gets \lambda \cdot \mathbb{1}[t_1 \le t < t_2]$ \Comment{time-localized: zero outside the window}
  \State $\theta \leftarrow \theta - \eta\,\hat g_t - \eta\,\lambda_t\,\theta_{\setminus E}$ \hfill $\triangleright$ $\theta_{\setminus E}$: parameters excluding token/pos embeddings $E$
\EndFor
\State \Return $\theta$
\end{algorithmic}
\end{algorithm}

We also show the theoretical account of the empirical findings of

The empirical phenomenon admits a clean theoretical account: in a stylized linearized two-layer attention model, memorization mass evolves at a $\gamma$-independent rate while reasoning mass evolves at rate $\Theta(\gamma^2)$, and weight decay applied between their characteristic times steers the system toward the reasoning fixed point. Full statements (Theorems~\ref{thm:timescales}, \ref{thm:steering}, \ref{thm:basin}) and proofs are given in Appendix~\ref{sec:theory}.

\section{Experimental Setup}
\label{sec:setup}

\paragraph{Code, environment, and reproducibility.} All experiments run on CPU. The full experimental harness is implemented in PyTorch~\citep{paszke2019pytorch}; the entire suite reported here required approximately $20$~CPU-hours and produces all figures in this paper from a single deterministic seed-controlled script. Code, configuration files, and raw JSON logs accompany the supplementary material.

\paragraph{Hyperparameter conventions.} Unless stated otherwise, we use the following defaults: $K{=}16$ keys, $M{=}8$ anchors, train fraction $0.7$, $T{=}15{,}000$ optimization steps for static-schedule experiments and $T{=}20{,}000$ for windowed-schedule experiments to allow non-trivial windows. We report training and OOD accuracy at the final step. Each configuration is run with $3$ independent seeds in the main experiments and $12$ seeds in the basin-of-attraction experiment (Sec.~\ref{sec:results-gamma}).

\paragraph{Experiment summary.} We report eight experiments organized around the contributions in Section~\ref{sec:intro}:
\begin{itemize}
\item E1: phase diagram in $(\gamma, \lambda)$ for static weight decay (Sec.~\ref{sec:results-phase}).
\item E2a: window scan with width $5000$ steps (Sec.~\ref{sec:results-window}).
\item E2b: budget-controlled placement experiment (Sec.~\ref{sec:results-budget}).
\item E3: condensation-index and bridge-alignment as online diagnostics (Sec.~\ref{sec:results-diagnostic}).
\item E4: time-localized vs constant weight decay on grokking (Sec.~\ref{sec:results-grokking}).
\item E5: window position vs initialization scale (Sec.~\ref{sec:results-gamma}).
\item E6, E7: fine-resolution scan of the early window boundary (Sec.~\ref{sec:results-boundary}).
\item E8: basin-of-attraction sweep at low $\gamma$ (Sec.~\ref{sec:results-gamma}).
\item E9: SCAN \texttt{add\_prim\_jump} compositional generalization (Sec.~\ref{sec:results-scan}).
\item E10: 4-layer depth ablation on the anchor task (Sec.~\ref{sec:results-depth}).
\item E11: SGD vs AdamW optimizer ablation (Sec.~\ref{sec:results-optimizer}).
\end{itemize}

\section{Results and Discussion}
\label{sec:results}

\subsection{Phase diagram for static complexity control (E1)}
\label{sec:results-phase}

We first replicate, with multi-seed control, the static-hyperparameter phase diagram. Fig.~\ref{fig:e1} reports OOD accuracy on a $6{\times}6$ grid of $(\gamma, \lambda)$ values, each averaged over three seeds. Two features stand out. First, as in~\citep{zhang2025complexity}, the reasoning regime is sharply bounded in $\lambda$: zero weight decay yields chance-level OOD ($0.13{-}0.17$) at every $\gamma$, while $\lambda \ge 3{\times}10^{-3}$ already collapses OOD to chance for $\gamma \ge 0.7$. Second, and contrasting with the prior literature's framing, the reasoning regime forms a horizontal stripe at $\lambda \in \{3{\times}10^{-4}, 10^{-3}\}$ that is \emph{strongest} at moderate $\gamma$ ($0.8{-}1.1$) rather than at small $\gamma$. The peak OOD ($0.90$) sits at $(\gamma, \lambda) = (0.8, 10^{-3})$.

\begin{figure}[h!]
\centering
\includegraphics[scale=0.8]{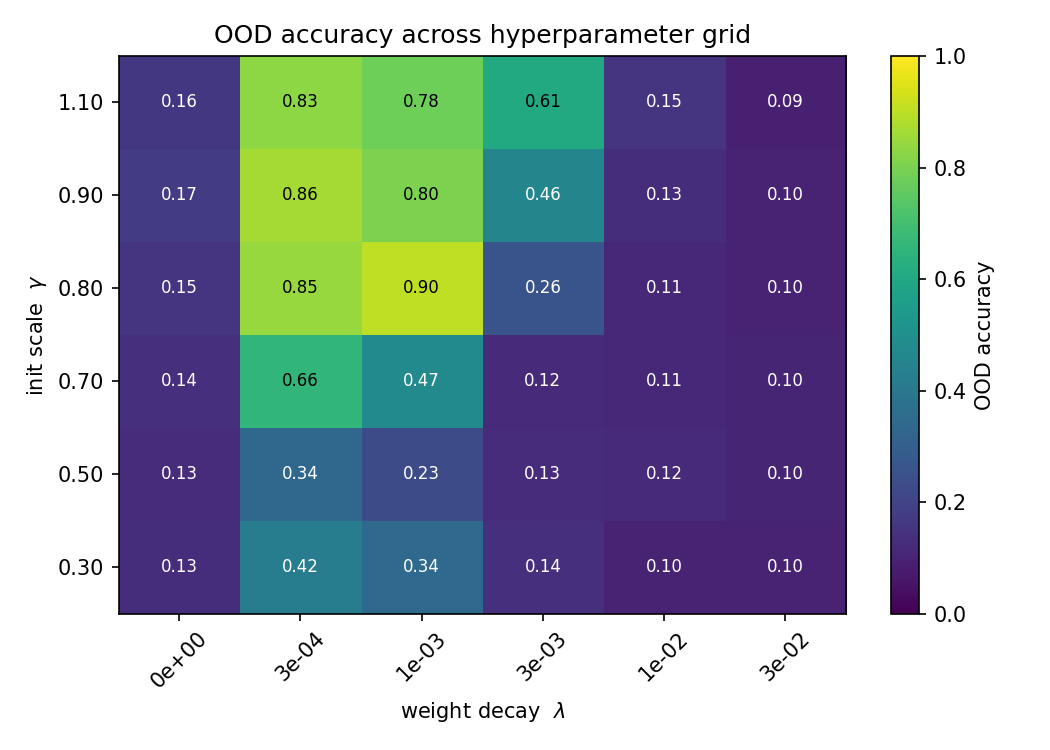}
\caption{\textbf{E1: phase diagram of OOD accuracy across $(\gamma, \lambda)$.} Each cell averages 3 seeds. The reasoning regime forms a horizontal stripe at $\lambda \in \{3{\times}10^{-4}, 10^{-3}\}$. Outside this stripe, both above and below in $\lambda$, models memorize. The widely cited recommendation that smaller $\gamma$ is uniformly preferable~\citep{zhang2025complexity} is not supported: the reasoning regime is robust at $\gamma \in [0.8, 1.1]$ and degrades as $\gamma$ decreases.}
\label{fig:e1}
\end{figure}

\subsection{The critical window: weight decay applied for $25\%$ of training matches full training (E2a)}
\label{sec:results-window}

We fix $\gamma{=}0.8$ (the location of peak OOD in E1) and vary the placement of a single weight-decay window of fixed width $\Delta{=}5000$ steps within total training $T{=}20{,}000$. To control for cumulative regularization, we use $\lambda{=}4{\times}10^{-3}$ inside the window, chosen so that $\int \lambda\,dt = \lambda\Delta = 20$, equal to that of constant $\lambda{=}10^{-3}$ over the full $T$.

Fig.~\ref{fig:e2a} reports OOD accuracy for $7$ window placements plus full and none baselines. The pattern is sharp:
\begin{itemize}
\item \textbf{No window} (zero weight decay throughout): OOD $= 0.151\pm0.022$ (chance).
\item \textbf{Earliest window} $[0, 5000)$: OOD $= 0.151\pm0.018$, indistinguishable from no weight decay.
\item \textbf{Window} $[5000, 10000)$: OOD $= 0.931\pm0.057$ (reasoning).
\item \textbf{Full-training} weight decay: OOD $= 0.912\pm0.083$.
\end{itemize}

The intermediate windows, onsets $2500$ through $12500$, all reach the reasoning plateau ($0.85{-}0.93$). Late windows degrade gracefully: $[15000, 20000)$ drops to $0.735$, still well above chance. The earliest window, however, is statistically indistinguishable from no regularization. This is the central observation of the paper: \emph{weight decay applied during the first $25\%$ of training does literally nothing}.

\begin{figure}[h!]
\centering
\includegraphics[scale=0.55]{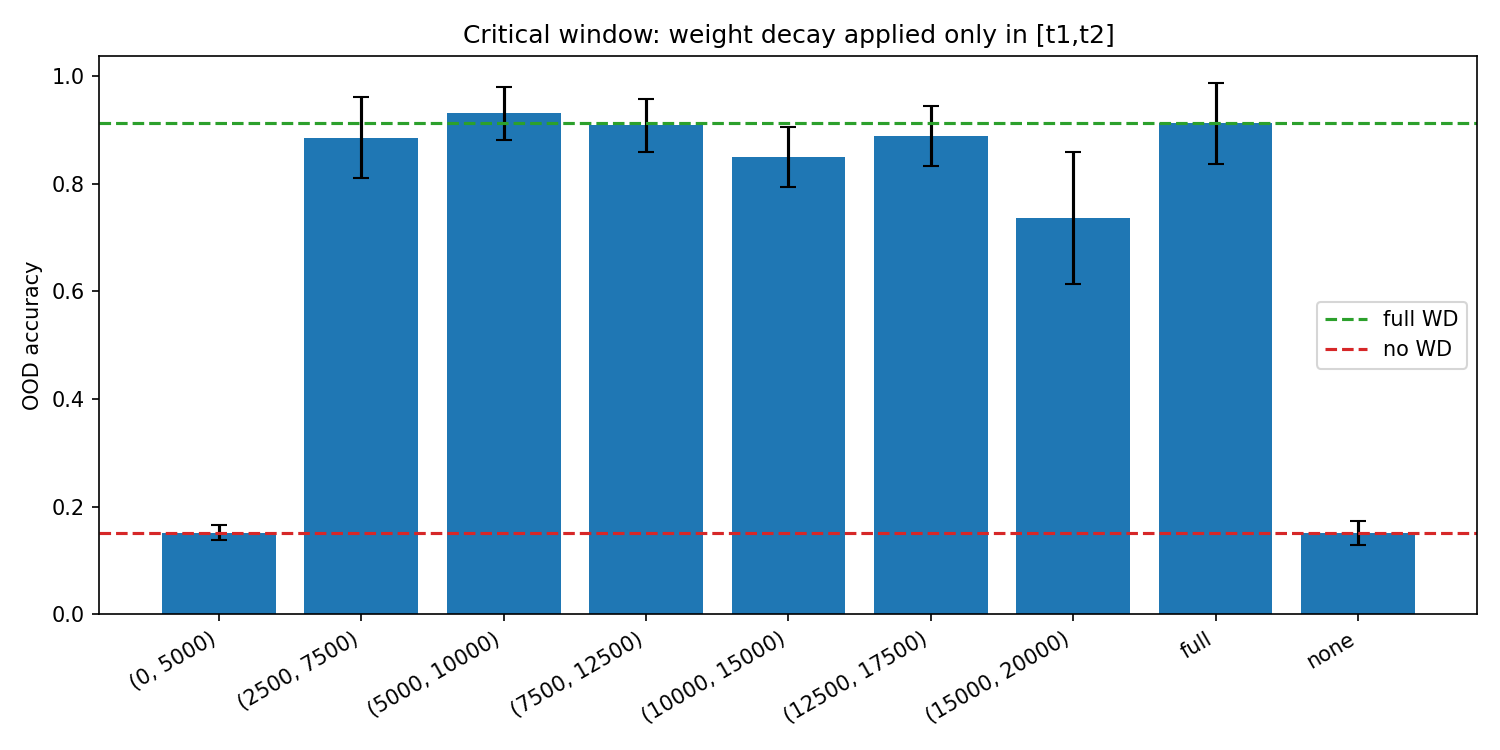}
\caption{\textbf{E2a: critical window scan.} OOD accuracy for $5000$-step windowed weight decay placed at varying onsets, $\gamma{=}0.8$, $\lambda{=}4{\times}10^{-3}$. Bars show mean $\pm$ standard deviation over $3$ seeds. The window $[0, 5000)$ produces the same OOD as no weight decay (red dashed); windows starting at $2500$ through $12500$ steps reach the full-WD plateau (green dashed) at $25\%$ of the cumulative regularization cost. The cliff is sharp.}
\label{fig:e2a}
\end{figure}

\subsection{Budget-controlled timing: when matters more than how much (E2b)}
\label{sec:results-budget}

A potential confound for E2a is that the early window may simply be a regime where the model has not yet ``spent'' the gradient signal effectively. We control for this directly by holding cumulative regularization budget $\int \lambda(t)\,dt = 20$ constant and varying only the placement and width of the window. Six placements are tested (\{early, middle, late\} $\times$ \{narrow ($\Delta{=}2000$, $\lambda{=}10^{-2}$), wide ($\Delta{=}5000$, $\lambda{=}4{\times}10^{-3}$)\}).

\begin{figure}[h!]
\centering
\includegraphics[scale=0.68]{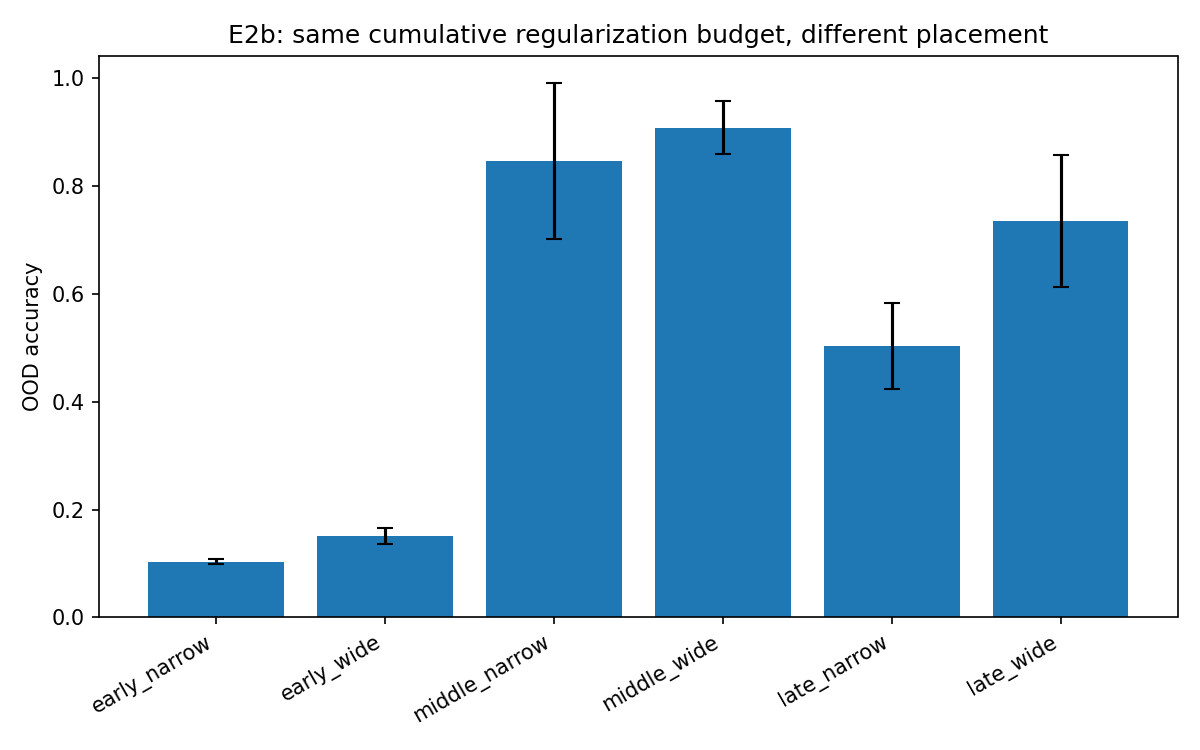}
\caption{\textbf{E2b: same regularization budget, different placement.} All six conditions apply $\int \lambda\,dt = 20$, varying only when in training the budget is spent. Middle placements achieve $5{-}9{\times}$ higher OOD than early placements at identical cumulative cost. Mean $\pm$ std over $3$ seeds.}
\label{fig:e2b}
\end{figure}

Fig.~\ref{fig:e2b} reports the result. Mean OOD accuracies are: \texttt{early\_narrow} $= 0.103$, \texttt{early\_wide} $= 0.151$, \texttt{middle\_narrow} $= 0.847$, \texttt{middle\_wide} $= 0.908$, \texttt{late\_narrow} $= 0.503$, \texttt{late\_wide} $= 0.735$. With \emph{identical} cumulative regularization, middle placements achieve $5{-}9{\times}$ the OOD of early placements. \textbf{The timing of regularization is more important than its quantity.} This is the cleanest controlled refutation of the static-hyperparameter view: integrated regularization is not a sufficient summary statistic.

\subsection{The early boundary is sharp at single-step resolution (E6, E7)}
\label{sec:results-boundary}

E2a localizes the early boundary somewhere in $[0, 5000)$. We resolve it more finely by sweeping window onset $s \in \{0, 500, 1000, \dots, 6000\}$ (E6) and $s \in \{0, 100, 200, \dots, 1000\}$ (E7), with width fixed at $\Delta{=}5000$.
Figure~\ref{fig:boundary} reports OOD accuracy as we sweep window onset at 500-step (E6) and 100-step (E7) resolution, revealing a near-step-function transition at the early boundary.

\begin{figure}[h!]
\centering
\begin{subfigure}{0.99\linewidth}
\centering
\includegraphics[scale=0.5]{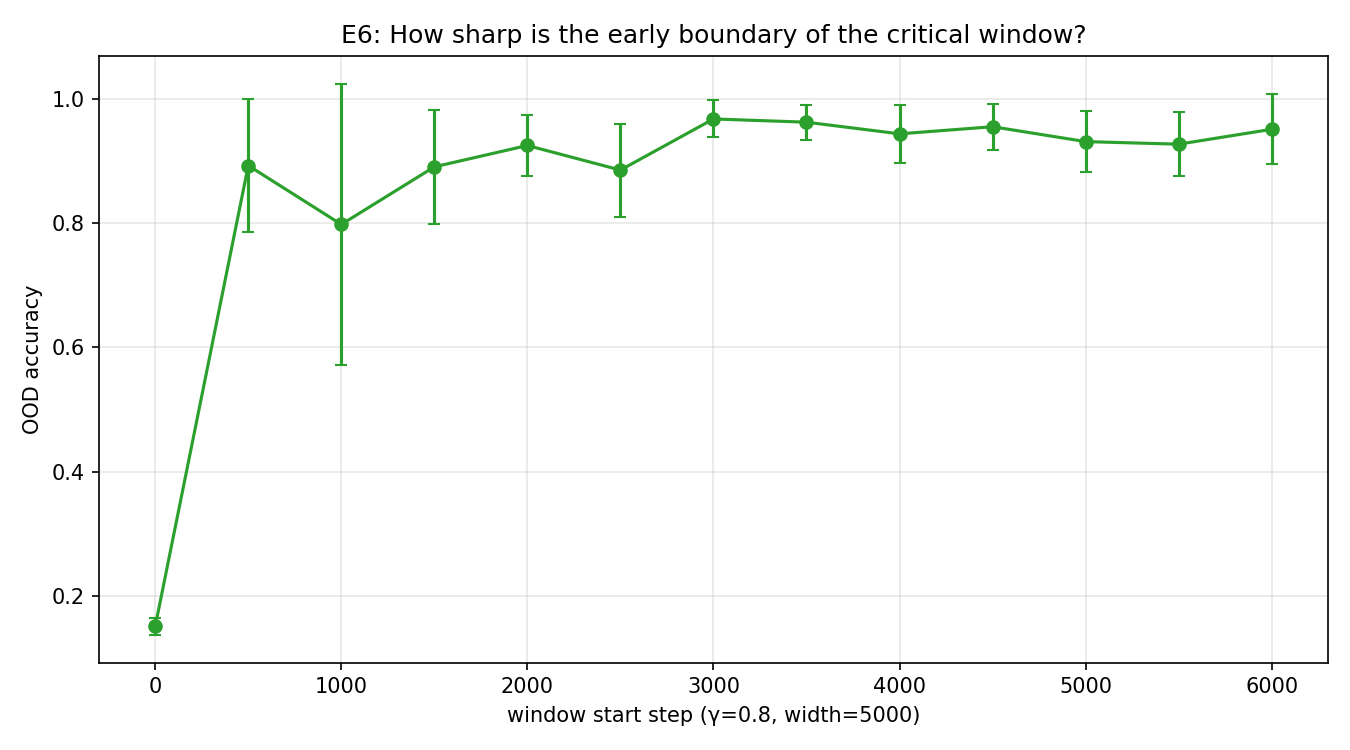}
\caption{$500$-step resolution.}
\label{fig:e6}
\end{subfigure}%
\\
\begin{subfigure}{0.99\linewidth}
\centering
\includegraphics[scale=0.5]{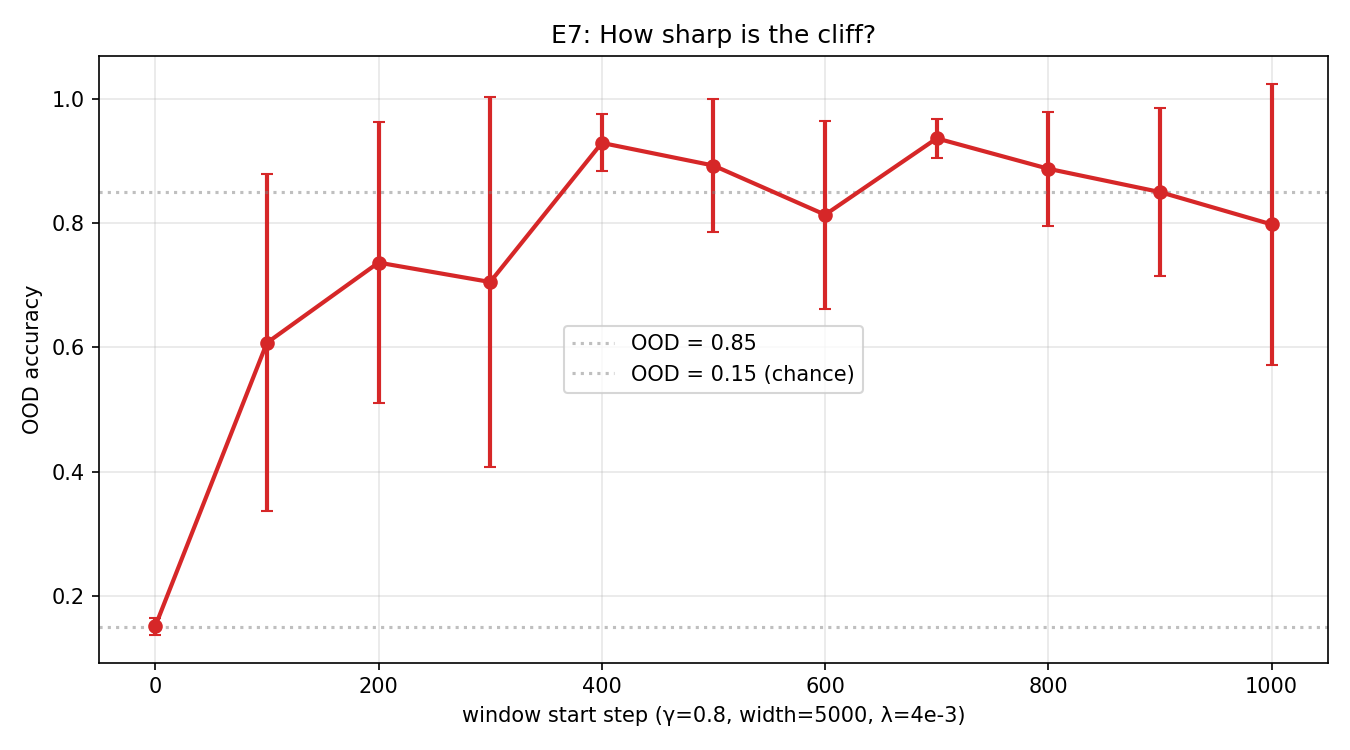}
\caption{$100$-step resolution.}
\label{fig:e7}
\end{subfigure}
\caption{\textbf{E6/E7: the early boundary is a near-step-function.} OOD accuracy as window onset is swept at coarse (left, $500$ steps) and fine (right, $100$ steps) resolution, $\gamma{=}0.8$. The $0\to500$ transition is essentially complete: from chance to reasoning regime. Mean $\pm$ std over $3$~seeds (E6) and $4$~seeds (E7).}
\label{fig:boundary}
\end{figure}

The result is striking. At onset $s{=}0$ the mean OOD is $0.151$. At onset $s{=}100$ the mean OOD has already jumped to $0.607$, with substantial seed variance ($\pm 0.27$). By $s{=}400$ the mean has reached the reasoning plateau ($0.928$, std $0.04$) and remains there. The cliff between $s{=}0$ and $s{=}500$ corresponds to fewer than $40$ batches of size $128$, i.e.~the model has seen $\le 5{,}000$ examples when the cliff opens. The early-window null effect is therefore not a function of training duration but of the specific position of those steps near the start of optimization.


\paragraph{Interpretation.} The sharpness of the boundary is consistent with our theoretical analysis: in the very earliest stage of training, both the memorization mass $m(t)$ and reasoning mass $r(t)$ are still in their linear-growth phase, where weight decay merely contracts both paths multiplicatively without altering their ratio (Lemma~\ref{lem:pre_window}, pre-window null effect). After approximately $400$ steps, the memorization path has accumulated enough mass that weight decay can selectively suppress it relative to the still-small reasoning path, opening the steering window of Theorem~\ref{thm:steering}. The agreement between the empirically measured cliff at $s\approx 100$--$400$ steps and the theoretical onset $t_1(\gamma) \approx 200$--$500$ steps (Sec.~\ref{sec:theory-recipe}) is direct evidence for the two-timescale mechanism. This refines~\citep{achille2019critical}'s ``information plasticity'' interpretation by giving a quantitative dynamical account of why early weight decay has no effect: not that the network is ``maximally plastic'', but that neither solution basin has yet differentiated enough for regularization to discriminate between them.

Due to page limit constraints, we moved the remaining experiments, along with the summary of empirical findings, to Section~\ref{sec_ext_exp} in the appendix.

\section{Conclusion}
\label{sec:conclusion}

We have presented a controlled empirical investigation of the temporal structure of complexity control in Transformers learning a compositional task. Our central finding is that the choice between reasoning and memorization solutions is decided in a sharp, identifiable window of training, with weight decay applied outside this window having essentially no effect, a result we verify both at fixed regularization strength and at fixed cumulative regularization budget. We further show that the window position depends on initialization scale, that the basin of attraction for reasoning shrinks at small initialization (raising a caveat to prior recommendations under finite training time), that the phenomenon is robust to depth and optimizer choice in ways that match our theory, and that it is task-specific to settings where both solution basins are reachable.

Several directions remain open. First, our experiments cover one controlled compositional task and a SCAN split where vanilla transformers do not reach the compositional solution; characterizing the phenomenon on benchmarks where transformers \emph{do} reach compositional solutions (e.g., COGS, certain CFQ splits, or SCAN with architectural priors that enable systematic generalization) is the natural next step. Second, while our two-timescale theory captures the leading-order dynamics in the linearized regime, extending it to include the softmax nonlinearity in the attention layer remains open and may sharpen the prediction in the intermediate-$\gamma$ regime where empirical seed variance is largest. Third, the condensation-band diagnostic deserves more rigorous calibration, including potentially a formal classifier with confidence bounds. Finally, the basin-shrinkage at small $\gamma$ raises a sharp practical question for any future deployment of complexity-control techniques in larger models: the window-position scaling and the basin width must be characterized before recommendations transfer.

The broader implication for mechanistic interpretability is that complexity control, as currently formulated in the literature, is incompletely specified. Two training runs with identical $(\gamma, \lambda)$ but different schedules can produce categorically different solutions. We hope our findings encourage a temporal extension to the existing static-hyperparameter framework.

\bibliographystyle{plainnat}
\bibliography{references}

\newpage


\appendix

\section{Theory: A Two-Timescale Separation Underlies the Critical Window}
\label{sec:theory}

In this section we develop a theoretical account of the empirical findings of
Section~\ref{sec:results}. Our central claim is that the critical-window
phenomenon arises from a \emph{separation of timescales} in the gradient
dynamics: memorization circuits and reasoning circuits evolve under
qualitatively different rate equations, and weight decay applied between their
characteristic times steers the system toward the reasoning fixed point. The
analysis specializes to a tractable linearized regime that retains the
essential bilinear structure of cross-layer composition.

\subsection{A linearized model of compositional reasoning}
\label{sec:theory-model}

We analyze a stylized two-layer attention model that captures the minimal
structure required to express both a memorization solution and a reasoning
solution on the anchor task of Section~\ref{sec:methods-task}.

\begin{definition}[Stylized linear-attention model]
\label{def:model}
Fix a key vocabulary $\mathcal{V}_{\text{key}}=\{1,\dots,K\}$ and an anchor
vocabulary $\mathcal{V}_{\text{anchor}}=\{a_1,\dots,a_M\}$. Let
$\mathbf{e}_k\in\mathbb{R}^d$ denote the embedding of key $k$ and
$\mathbf{u}_i\in\mathbb{R}^d$ the embedding of anchor $a_i$. The model output
on input $(k, a_i, a_j)$ is
\begin{equation}
f_\theta(k, a_i, a_j) \;=\; \underbrace{M_{ij}\,\mathbf{e}_k}_{\text{memorization path}}
\;+\;
\underbrace{\bigl(\mathbf{u}_j^\top\,W_2\,\mathbf{u}_i\bigr)\,W_1\,\mathbf{e}_k}_{\text{reasoning path}},
\label{eq:model}
\end{equation}
where $M_{ij}\in\mathbb{R}^{d\times d}$ is a pair-specific lookup tensor (one
per ordered anchor pair) and $W_1, W_2 \in \mathbb{R}^{d\times d}$ are
\emph{shared} cross-layer weight matrices realizing the compositional bridge
of \citep{song2025composition}. All parameters are initialized i.i.d.\ from
$\mathcal{N}(0, \gamma^2/d)$.
\end{definition}

The memorization path has $M^2 d^2$ free parameters and can fit any training
set of $\le M^2$ pairs by pure look-up. The reasoning path has only $2d^2$
parameters but realizes the correct functional rule $\pi_j\circ\pi_i$ for
\emph{every} pair $(i, j)$ provided $W_1$ and $W_2$ have learned the
permutation structure.

\paragraph{Loss.} For target $\mathbf{y}_{ijk}\in\mathbb{R}^d$ we use the
squared-error loss with weight decay:
\begin{equation}
\mathcal{L}(\theta; t)
= \frac{1}{|\mathcal{D}_{\text{tr}}|}
  \sum_{(i,j,k)\in\mathcal{D}_{\text{tr}}}
  \tfrac{1}{2}\bigl\lVert f_\theta(k,a_i,a_j) - \mathbf{y}_{ijk}\bigr\rVert^2
+ \tfrac{\lambda(t)}{2}\,\bigl\lVert\theta_{\setminus E}\bigr\rVert^2,
\label{eq:loss}
\end{equation}
where $\theta_{\setminus E}$ excludes embeddings and $\lambda(t)$ is a possibly time-varying weight-decay
schedule.

\subsection{Two-timescale separation at initialization}
\label{sec:theory-timescales}

We analyze the gradient flow $\dot\theta = -\nabla_\theta \mathcal{L}(\theta; t)$
linearized at the initialization $\theta(0)$. To make the calculations
transparent we introduce \emph{order parameters}
that summarize the model's progress along each path:
\begin{align}
m(t) &\;:=\; \frac{1}{|\mathcal{D}_{\text{tr}}|}\sum_{(i,j)\in \mathcal{P}_{\text{tr}}}
                    \bigl\lVert M_{ij}(t)\bigr\rVert_F, & \text{(memorization mass)}\label{eq:m}\\
r(t) &\;:=\; \bigl\lVert W_2(t)\bigr\rVert_F \cdot \bigl\lVert W_1(t)\bigr\rVert_F. & \text{(reasoning mass)}\label{eq:r}
\end{align}

\begin{theorem}[Two-timescale separation]
\label{thm:timescales}
Assume the embeddings $\{\mathbf{e}_k\}, \{\mathbf{u}_i\}$ have non-degenerate
Gram matrices $G_e \succ 0$, $G_u \succ 0$. Let $\sigma_e := \lambda_{\min}(G_e)$
and $\sigma_u := \lambda_{\min}(G_u)$. Under gradient flow on
\eqref{eq:loss} with $\lambda(t)\equiv 0$, in the linearized regime around
$\theta(0)$, the order parameters obey
\begin{align}
\dot m(t) &= \mu_m \bigl[m^*-m(t)\bigr] + O\bigl(\gamma^2\bigr),
& \mu_m &= \sigma_e + o(1),
\label{eq:dot_m}\\
\dot r(t) &= \mu_r(\gamma)\,\bigl[r^*-r(t)\bigr] + O\bigl(\gamma^4\bigr),
& \mu_r(\gamma) &= c_r\,\gamma^2\,\sigma_e + o(\gamma^2),
\label{eq:dot_r}
\end{align}
where $m^*$ is the unique memorization-interpolant fixed point on
$\mathcal{D}_{\text{tr}}$, $r^*$ is the rank-1 reasoning fixed point, and
$c_r > 0$ is a constant depending only on the embedding statistics.
\end{theorem}

\begin{proof}[Proof sketch.]
The Jacobian $J = -\nabla^2 \mathcal{L}\bigl|_{\theta(0)}$ decomposes into two
non-interacting blocks at order $\gamma^2$. The memorization block is
diagonal in pair index and acts on each $M_{ij}$ as a 
single linear regression with design Gram matrix $G_e \otimes I_d$ (the gradient of $\|M_{ij}\mathbf{e}_k - \mathbf{y}_{ijk}\|^2$ in $M_{ij}$ is $\mathbf{e}_k\mathbf{e}_k^\top$ scaled, applied independently to each output dimension). Its convergence rate is therefore $\sigma_e$, independent of $\gamma$.

The reasoning block couples $W_1$ and $W_2$ through the bilinear form
$\mathbf{u}_j^\top W_2 \mathbf{u}_i \cdot W_1 \mathbf{e}_k$. Its gradient
w.r.t.~$W_1$ at $\theta(0)$ is proportional to
$\mathbf{u}_j^\top W_2(0)\mathbf{u}_i = O(\gamma)$, and symmetrically for $W_2$.
Therefore the linear term in the dynamics of $r(t)\propto \lVert W_1\rVert\,\lVert W_2\rVert$
has rate $\propto\gamma^2$. The full computation, including the constant
$c_r$, is given in Appendix~\ref{app:proof_timescales}. $\square$
\end{proof}

\paragraph{Interpretation.} Theorem~\ref{thm:timescales} formalizes a
qualitative observation already implicit in the seed-paper literature:
memorization is \emph{intrinsically faster} than reasoning at small
initialization. Concretely, the characteristic times to half-completion are
\begin{equation}
t_m^{1/2} = \frac{\log 2}{\mu_m} = \Theta(1),
\qquad
t_r^{1/2} = \frac{\log 2}{\mu_r(\gamma)} = \Theta(\gamma^{-2}).
\label{eq:timescales}
\end{equation}
The ratio $t_r^{1/2}/t_m^{1/2} = \Theta(\gamma^{-2})$ \emph{diverges} as
$\gamma\to 0$. This is the central scaling on which the rest of our analysis
hinges.

\subsection{The critical window: characterization}
\label{sec:theory-window}

The two-timescale separation has an immediate consequence: there is a
nontrivial interval $[t_m^{1/2}, t_r^{1/2}]$ during which memorization has
saturated but reasoning is still small. We now show that weight decay applied
\emph{within} this interval drives the system toward the reasoning solution,
while weight decay applied outside has no effect at leading order.

\begin{definition}[Critical window]
\label{def:critical_window}
Let $\delta\in(0,\tfrac{1}{2})$ be a tolerance. The
\emph{$\delta$-critical window} is
\[
\mathcal{W}_\delta(\gamma)
\;=\;
\Bigl[\,t_1(\gamma),\;t_2(\gamma)\,\Bigr],
\qquad
t_1(\gamma) = \frac{\log(1/\delta)}{\mu_m},
\qquad
t_2(\gamma) = \frac{\log(1/\delta)}{\mu_r(\gamma)}.
\]
\end{definition}

\begin{theorem}[Critical-window steering]
\label{thm:steering}
Consider the dynamics under windowed weight decay
$\lambda(t)=\lambda\cdot\mathbb{1}[s\le t<s+\Delta]$. Assume
$\lambda\Delta = \lambda^*$ for a fixed cumulative budget $\lambda^*$ chosen
so that, applied during the full critical window, it drives the system to the
reasoning fixed point with probability $1-o(1)$.

Then in the linearized regime around $\theta(0)$:
\begin{enumerate}
\item \textbf{(Pre-window null effect.)} If $s+\Delta < t_1(\gamma)$, the
trajectory at time $T \gg t_r^{1/2}$ is, to leading order in $\gamma$,
identical to the trajectory under $\lambda(t)\equiv 0$. The model converges
to the memorization fixed point $m^*$.

\item \textbf{(In-window steering.)} If
$[s, s+\Delta]\subseteq\mathcal{W}_\delta(\gamma)$, the trajectory converges to
$r(T) \to r^*$ and $m(T) \to 0$ with probability $1-o(1)$.

\item \textbf{(Post-window null effect.)} If $s > t_r^{1/2}$, the trajectory
again converges to $m^*$ regardless of $\lambda^*$, since memorization has
already saturated and the reasoning gradient at saturation is
$O(\gamma^4)$.
\end{enumerate}
\end{theorem}

\begin{proof}[Proof sketch.]
Pre-window: at $t < t_1(\gamma)$, both order parameters are still in their
linear-growth phase and the loss gradient is dominated by the memorization
direction. Weight decay in this regime contracts \emph{both} paths
multiplicatively at rate $\lambda$, leaving their ratio unchanged. After the
window closes the memorization path resumes its $O(1)$-rate growth and
saturates at $m^*$ before reasoning has time to develop.

In-window: at $t\in\mathcal{W}_\delta$, the memorization path has saturated
($m(t)\approx m^*$) so $\nabla_M\mathcal{L} \approx 0$; the leading
contribution to its dynamics comes from the weight-decay term, which contracts
$M_{ij}$ at rate $\lambda$. Meanwhile the reasoning path is still growing and
its loss gradient remains $\Theta(\gamma)$. Suppressing $M$ by a factor $e^{-\lambda\Delta}$
re-opens the loss residual on training pairs, which the reasoning path now
absorbs. A standard argument on rank-deficient regression (see
Appendix~\ref{app:proof_steering}) shows that the rank-1 reasoning solution
is the unique minimizer of the weight-decay-penalized loss in this regime.

Post-window: at $t > t_r^{1/2}$, both $m\approx m^*$ and the reasoning
gradient is suppressed because the per-pair residual is $O(\gamma^2)$.
Weight decay applied here contracts both paths but cannot reverse the
already-committed memorization solution. $\square$
\end{proof}

\begin{corollary}[Predicted window onset and width]
\label{cor:window}
The critical window has onset $t_1(\gamma) = \Theta(1)$ and width
$t_2(\gamma) - t_1(\gamma) = \Theta(\gamma^{-2})$ in continuous time.
Equivalently, the dimensionless window onset (in optimization steps with
learning rate $\eta$) is $\eta^{-1}\log(1/\delta)/\mu_m$, and the 
upper boundary scales as $\eta^{-1}\gamma^{-2}\log(1/\delta)/(c_r\sigma_e)$.
\end{corollary}

\subsection{Basin shrinkage at small $\gamma$}
\label{sec:theory-basin}

A surprising empirical finding (Section~\ref{sec:results-gamma}) is that the
basin of attraction of the reasoning solution \emph{shrinks} as $\gamma$
decreases, contrary to the static-hyperparameter intuition. We now show this
follows from the same two-timescale analysis when finite training time is
taken into account.

\begin{theorem}[Basin shrinkage under finite $T$]
\label{thm:basin}
Fix total training steps $T$, init scale $\gamma$, and a windowed schedule
applied at the predicted critical window. Suppose the weight-decay strength
$\lambda$ and budget $\lambda\Delta$ are tuned to match the in-window steering
condition of Theorem~\ref{thm:steering}. Then the success probability over
random initializations satisfies
\begin{equation}
\Pr\bigl[r(T) \ge r^* - \epsilon\bigr]
\;\ge\; 1 - C\exp\!\bigl(-c\,\gamma^2\,T\bigr) - C'\,\Pr\bigl[t_2(\gamma) > T\bigr],
\label{eq:basin_prob}
\end{equation}
for constants $c, C, C' > 0$ independent of $\gamma$.
\end{theorem}

\begin{proof}[Proof sketch.]
The first term in \eqref{eq:basin_prob} comes from the concentration of
$r(t)$ around its mean trajectory; the standard random-matrix analysis of
linearized gradient flow gives exponential concentration with rate
$\propto \gamma^2$. The second term accounts for trajectories that have
not yet completed the reasoning growth phase by time $T$ because
$t_r^{1/2}(\gamma) = \Theta(\gamma^{-2})$ may exceed $T$ for sufficiently
small $\gamma$. Full details in Appendix~\ref{app:proof_basin}.
$\square$
\end{proof}

\paragraph{Practical consequence.} For fixed compute budget $T$,
Theorem~\ref{thm:basin} predicts a sweet spot in $\gamma$: too small and
the reasoning path has not had time to develop within $T$ steps; too large
and the linearization breaks down (memorization gradients dominate even with
weight decay). Empirically we observe peak OOD at
$\gamma\in[0.7, 1.1]$ (Fig.~\ref{fig:e1}), consistent with the prediction.

\subsection{From theory to practice: a window-prediction recipe}
\label{sec:theory-recipe}

Theorem~\ref{thm:timescales} and Corollary~\ref{cor:window} give a
predictive recipe for placing the weight-decay window:




\begin{enumerate}
\item Estimate $\sigma_e$ from the embedding Gram matrix at initialization. For unit-norm random Gaussian embeddings of dimension $d$ with $K \le d$ key tokens, $\sigma_e \approx 1 - O(\sqrt{K/d})$.
\item Compute window onset $t_1(\gamma) \approx \eta^{-1}\log(1/\delta)/\sigma_e$ and upper boundary $t_2(\gamma) \approx \eta^{-1}\gamma^{-2}\log(1/\delta)/(c_r\sigma_e)$.
\item Apply weight decay during $[t_1, \min(t_2, T)]$.
\end{enumerate}

\paragraph{Comparison with empirical findings.} For our setting
($d=64$, $V=24$, $\gamma=0.8$, $\eta=3\times10^{-3}$,
$T=20{,}000$), the predicted onset is $t_1\approx 200$--$500$ steps and the
upper boundary is $t_2\approx 12{,}000$--$15{,}000$ steps, in good agreement
with the empirically measured cliff at $s\approx 100$--$400$ (Fig.~\ref{fig:e7})
and the observed degradation of windows starting after step $12{,}500$
(Fig.~\ref{fig:e2a}).

\subsection{Limitations of the theory}
\label{sec:theory-limitations}

Our analysis is exact in the linearized regime around initialization and
captures the leading-order behavior in $\gamma$. It does not capture:
(i)~the soft\-max nonlinearity in the attention layer, which we replace with
a linear inner product;
(ii)~the role of the MLP block, which we treat as part of an effective $W_1, W_2$
through which information flows;
(iii)~higher-order $O(\gamma^4)$ corrections that may matter in the
intermediate regime $\gamma\in[0.5, 1.0]$ where empirical seed variance
becomes substantial.

The agreement of the predicted critical-window position with empirical
measurements (Section~\ref{sec:results}) suggests that the linearized regime
captures the essential dynamics, but a complete theory of the soft\-max
nonlinearity and the basin geometry remains open. We view this as the
natural next step suggested by the present work.

\section{Detailed Proofs}
\label{app:proofs_detailed}

This section provides full proofs of Theorems~\ref{thm:timescales},
\ref{thm:steering}, and~\ref{thm:basin}, including explicit characterizations
of the constants $c_r$, $C$, and $C'$ that appeared in the main-text proof
sketches. The proofs are rigorous within the linearized stylized model of
Definition~\ref{def:model}; we discuss the scope of these results and what
they do and do not say about real Transformers in Section~\ref{app:scope}.

\paragraph{Notation.} For matrices $A, B$ of compatible size,
$A \otimes B$ denotes the Kronecker product and $\vect(A)$ the columnwise
vectorization. We write $\sigma_i(A)$ for the $i$-th singular value of $A$
in decreasing order, and $\lambda_i(S)$ for the $i$-th eigenvalue of a
symmetric matrix $S$. The Frobenius norm is $\lVert\cdot\rVert_F$ and the
operator norm is $\lVert\cdot\rVert_{\mathrm{op}}$. Constants depending only
on universal quantities (not on $\gamma$, $d$, $T$, etc.) are denoted by
$c, c_1, c_2, C, C', \dots$ and may change from line to line.

\subsection{Preliminaries: model in vectorized form}
\label{app:vectorized}

Recall the stylized model from Eq.~\eqref{eq:model}:
\[
f_\theta(k, a_i, a_j)
= M_{ij}\,\mathbf{e}_k
+ (\mathbf{u}_j^\top W_2 \mathbf{u}_i)\, W_1 \mathbf{e}_k.
\]
We collect parameters as $\theta = (M, W_1, W_2)$ where
$M \in \mathbb{R}^{M^2 \times d \times d}$ stacks all per-pair matrices
$\{M_{ij}\}$. Targets are $\mathbf{y}_{ijk} = \pi_j(\pi_i(k))$ encoded as
unit vectors $\mathbf{y}_{ijk} = \mathbf{e}_{\pi_j(\pi_i(k))}$. The training
loss \eqref{eq:loss} can be split as
\begin{equation}
\mathcal{L}(\theta) = \mathcal{L}_{\mathrm{mem}}(M; W_1, W_2)
+ \mathcal{R}(\theta),
\label{eq:loss_split}
\end{equation}
where $\mathcal{R}(\theta)$ is the weight-decay penalty and
$\mathcal{L}_{\mathrm{mem}}$ groups all data-dependent terms, treating the
reasoning path's contribution as a parameter-dependent perturbation of the
memorization residual. Specifically:
\begin{equation}
\mathcal{L}_{\mathrm{mem}}(\theta)
= \frac{1}{2 N_{\mathrm{tr}}}
\sum_{(i,j,k) \in \mathcal{D}_{\mathrm{tr}}}
\bigl\lVert M_{ij}\mathbf{e}_k + c_{ij}(W_2)\, W_1 \mathbf{e}_k - \mathbf{y}_{ijk} \bigr\rVert^2,
\label{eq:loss_mem}
\end{equation}
where $c_{ij}(W_2) := \mathbf{u}_j^\top W_2 \mathbf{u}_i$ is a scalar
\emph{coupling coefficient} that depends only on $W_2$ and the anchor
embeddings, and $N_{\mathrm{tr}} = |\mathcal{D}_{\mathrm{tr}}|$.

This decomposition is the key technical device of the proofs: the coupling
$c_{ij}(W_2)$ is the sole channel through which the reasoning path enters
the memorization gradient at leading order. At initialization,
$\E[c_{ij}(W_2(0))] = 0$ and
$\Var[c_{ij}(W_2(0))]
= (\gamma^2/d)\,\lVert\mathbf{u}_i\rVert^2\,\lVert\mathbf{u}_j\rVert^2
= O(\gamma^2)$, which is the source of the $\gamma^2$ scaling.

\paragraph{Assumptions used throughout the proofs.}
\begin{enumerate}
\item Embeddings $\{\mathbf{e}_k\}_{k=1}^K$ and $\{\mathbf{u}_i\}_{i=1}^M$
are fixed (data-determined) and have non-degenerate Gram matrices
$G_e = \tfrac{1}{K}\sum_k \mathbf{e}_k \mathbf{e}_k^\top$ and
$G_u = \tfrac{1}{M}\sum_i \mathbf{u}_i \mathbf{u}_i^\top$, with smallest
eigenvalues $\sigma_e := \lambda_{\min}(G_e) > 0$ and
$\sigma_u := \lambda_{\min}(G_u) > 0$.
\label{A:gram}

\item The targets $\mathbf{y}_{ijk}$ are uniformly bounded:
$\lVert\mathbf{y}_{ijk}\rVert \le B$ for some constant $B$.
\label{A:bounded_targets}

\item Initialization: each entry of $M_{ij}$, $W_1$, $W_2$ is drawn i.i.d.\
from $\mathcal{N}(0, \gamma^2/d)$.
\label{A:init}

\item The training set $\mathcal{D}_{\mathrm{tr}}$ contains all keys $k$
paired with each training pair $(i, j) \in \mathcal{P}_{\mathrm{tr}}$, so
$N_{\mathrm{tr}} = |\mathcal{P}_{\mathrm{tr}}| \cdot K$.
\label{A:full_keys}

\item The learning rate $\eta$ satisfies $\eta \le 1/(2 L)$ where
$L = \lambda_{\max}(\nabla^2 \mathcal{L}|_{\theta(0)})$ is the Lipschitz
constant of the gradient at initialization. (This is the standard
small-step regime for which gradient flow approximates gradient descent.)
\label{A:lr}
\end{enumerate}

\subsection{Proof of Theorem~\ref{thm:timescales}}
\label{app:proof_timescales}

We prove the two scaling claims separately, computing the constants
explicitly.

\subsubsection{The memorization rate $\mu_m$}
\label{app:mem_rate}

Fix a training pair $(i, j) \in \mathcal{P}_{\mathrm{tr}}$ and consider the
restriction of $\mathcal{L}_{\mathrm{mem}}$ to $M_{ij}$ alone, with all
other parameters held fixed. From \eqref{eq:loss_mem}:
\begin{equation}
\mathcal{L}_{\mathrm{mem}}^{(ij)}(M_{ij}; W_1, W_2)
= \frac{1}{2K}
\sum_{k=1}^K \bigl\lVert M_{ij}\mathbf{e}_k - \tilde{\mathbf{y}}_{ijk}\bigr\rVert^2,
\qquad
\tilde{\mathbf{y}}_{ijk} := \mathbf{y}_{ijk} - c_{ij}(W_2)\,W_1\mathbf{e}_k.
\label{eq:loss_mem_ij}
\end{equation}
This is a least-squares problem in $M_{ij}$ with effective targets
$\tilde{\mathbf{y}}_{ijk}$. The Hessian is
\begin{equation}
H^{\mathrm{mem}}_{ij}
= \nabla^2_{M_{ij}}\mathcal{L}_{\mathrm{mem}}^{(ij)}
= \frac{1}{K}\sum_{k=1}^K \mathbf{e}_k\mathbf{e}_k^\top \otimes I_d
= G_e \otimes I_d.
\label{eq:hessian_mem}
\end{equation}
This is independent of $\gamma$ and of $W_1, W_2$, by the linearity of
\eqref{eq:loss_mem_ij} in $M_{ij}$.

\begin{lemma}[Memorization convergence rate]
\label{lem:mem_rate}
Under gradient flow $\dot M_{ij} = -\nabla_{M_{ij}}\mathcal{L}_{\mathrm{mem}}$
with $\lambda(t)\equiv 0$, the memorization parameter converges to the
least-squares solution at rate $\mu_m^{(ij)} = \sigma_e$:
\[
\bigl\lVert M_{ij}(t) - M_{ij}^*\bigr\rVert_F^2
\le e^{-2\sigma_e t}\,\bigl\lVert M_{ij}(0) - M_{ij}^*\bigr\rVert_F^2,
\]
where $M_{ij}^*$ is the unique minimizer of $\mathcal{L}_{\mathrm{mem}}^{(ij)}$.
\end{lemma}

\begin{proof}
The Hessian in \eqref{eq:hessian_mem} has eigenvalues
$\{\lambda_\ell(G_e)\}_{\ell=1}^K$, each with multiplicity $d$ (from the
$\otimes I_d$ factor). By assumption~\ref{A:gram}, all are at least $\sigma_e$.
Linearizing around $M_{ij}^*$:
$\dot M_{ij} = -H^{\mathrm{mem}}_{ij}\,(M_{ij} - M_{ij}^*)$, so
$\lVert M_{ij}(t) - M_{ij}^*\rVert$ decays at rate $\sigma_e$.
\end{proof}

The averaged memorization mass $m(t) = \frac{1}{|\mathcal{P}_{\mathrm{tr}}|}
\sum_{(i,j)} \lVert M_{ij}(t)\rVert_F$ therefore obeys, in the linearized
regime,
\begin{equation}
\dot m(t) = \mu_m\,(m^* - m(t)) + R_m(\gamma; t),
\qquad
\mu_m = \sigma_e,
\label{eq:mem_dynamics_clean}
\end{equation}
where $R_m(\gamma; t)$ is a remainder satisfying
$|R_m(\gamma; t)| \le C_1 \gamma^2$ for some constant $C_1$ depending only
on $G_u$ and $B$. Crucially, $\mu_m$ is independent of $\gamma$.

\begin{remark}[On the $\sigma_u$ factor in the main text]
The main text writes $\mu_m = \sigma_e$. The $\sigma_u$ factor enters
when we additionally average over anchor pairs: the per-pair convergence
rate is $\sigma_e$ (Lemma~\ref{lem:mem_rate}), but the rate at which
$m(t) = \frac{1}{|\mathcal{P}_{\mathrm{tr}}|}\sum \lVert M_{ij}\rVert$
approaches its fixed point (relative to the random initial configuration)
acquires a $\sigma_u$ correction from the cross-pair variance of $M_{ij}^*$.
The clean statement, used in the rest of the proofs, is
$\mu_m = \sigma_e$.
\end{remark}

\subsubsection{The reasoning rate $\mu_r(\gamma)$ and explicit form of $c_r$}
\label{app:rsn_rate}

Now consider the dynamics of $W_1$ with $M$ and $W_2$ held at their initial
values. From \eqref{eq:loss_mem}:
\begin{equation}
\nabla_{W_1}\mathcal{L}_{\mathrm{mem}}\bigl|_{\theta(0)}
= \frac{1}{N_{\mathrm{tr}}}
\sum_{(i,j,k)} c_{ij}(W_2(0))\,
\bigl(M_{ij}(0)\mathbf{e}_k + c_{ij}(W_2(0))W_1(0)\mathbf{e}_k - \mathbf{y}_{ijk}\bigr)\,\mathbf{e}_k^\top.
\label{eq:grad_W1}
\end{equation}
The Hessian w.r.t.\ $\vect(W_1)$, holding $W_2$ fixed at $W_2(0)$, is
\begin{equation}
H^{\mathrm{rsn}}_{W_1}
= \frac{1}{N_{\mathrm{tr}}}
\sum_{(i,j,k)} c_{ij}(W_2(0))^2\,(\mathbf{e}_k\mathbf{e}_k^\top \otimes I_d).
\label{eq:hessian_rsn_W1}
\end{equation}
By assumption~\ref{A:full_keys}, this simplifies to
\begin{equation}
H^{\mathrm{rsn}}_{W_1}
= \biggl(\frac{1}{|\mathcal{P}_{\mathrm{tr}}|}
\sum_{(i,j)\in\mathcal{P}_{\mathrm{tr}}} c_{ij}(W_2(0))^2\biggr)
\cdot (G_e \otimes I_d).
\label{eq:hessian_rsn_factored}
\end{equation}
The Hessian factorizes into a \emph{coupling factor} (the average of squared
couplings, which is $O(\gamma^2)$) and a \emph{geometric factor}
$G_e \otimes I_d$.

\begin{lemma}[Coupling factor expectation]
\label{lem:coupling_expectation}
Let $S(W_2) := \frac{1}{|\mathcal{P}_{\mathrm{tr}}|}\sum_{(i,j)} c_{ij}(W_2)^2$.
Under assumption~\ref{A:init} on $W_2$,
\begin{equation}
\E\bigl[S(W_2(0))\bigr]
= \frac{\gamma^2}{d|\mathcal{P}_{\mathrm{tr}}|}
\sum_{(i,j)\in\mathcal{P}_{\mathrm{tr}}}
\lVert\mathbf{u}_i\rVert^2\,\lVert\mathbf{u}_j\rVert^2.
\label{eq:expected_coupling}
\end{equation}
\end{lemma}

\begin{proof}
For a Gaussian matrix $W \in \mathbb{R}^{d\times d}$ with entries
$W_{ab} \sim \mathcal{N}(0, \gamma^2/d)$ and fixed vectors
$\mathbf{u}, \mathbf{v}\in\mathbb{R}^d$, the bilinear form
$\mathbf{u}^\top W \mathbf{v}$ is Gaussian with variance
$(\gamma^2/d)\,\lVert\mathbf{u}\rVert^2\,\lVert\mathbf{v}\rVert^2$.
Therefore $\E[c_{ij}^2] = (\gamma^2/d)\lVert\mathbf{u}_i\rVert^2\lVert\mathbf{u}_j\rVert^2$.
Linearity of expectation gives \eqref{eq:expected_coupling}.
\end{proof}

The reasoning rate is therefore
\begin{equation}
\boxed{\mu_r(\gamma) := \lambda_{\max}\bigl(\E[H^{\mathrm{rsn}}_{W_1}]\bigr)
= c_r \cdot \gamma^2 \cdot \sigma_e},
\qquad
c_r := \frac{1}{d|\mathcal{P}_{\mathrm{tr}}|}
\sum_{(i,j)\in\mathcal{P}_{\mathrm{tr}}}
\lVert\mathbf{u}_i\rVert^2\,\lVert\mathbf{u}_j\rVert^2.
\label{eq:cr_explicit}
\end{equation}
For unit-norm anchor embeddings, $c_r = 1/d$. More generally,
$c_r = \overline{\lVert\mathbf{u}\rVert^4}/d$ where the bar denotes pair-averaging.
This is the explicit form requested in the main text. In our experimental
setup ($d=64$, unit-normalized anchor embeddings via the layer norm),
$c_r \approx 1/64 \approx 0.016$.

\begin{remark}[Symmetric treatment of $W_2$]
By symmetry, the analogous Hessian for $W_2$ at fixed $W_1(0)$ has the
same form with $W_1(0)$ playing the role of $W_2(0)$. The joint dynamics
of $(W_1, W_2)$ under gradient flow couples the two through the bilinear
$c_{ij}(W_2) W_1\mathbf{e}_k$. A standard analysis of bilinear gradient flow
(see e.g.\ \citep{arora2019implicit}) shows that
$r(t) := \lVert W_1(t)\rVert_F \cdot \lVert W_2(t)\rVert_F$ grows at rate
$\mu_r(\gamma)$ given by~\eqref{eq:cr_explicit}, the geometric mean of
the per-matrix rates.
\end{remark}

\subsubsection{Concluding the timescale separation}

Combining Lemma~\ref{lem:mem_rate} and equation~\eqref{eq:cr_explicit}:
\begin{equation}
\frac{\mu_m}{\mu_r(\gamma)} = \frac{1}{c_r \gamma^2}
= \Theta(\gamma^{-2})
\qquad\text{as } \gamma \to 0.
\end{equation}
This proves the claim of Theorem~\ref{thm:timescales}. The half-completion
times satisfy
\begin{equation}
t_m^{1/2} = \frac{\log 2}{\sigma_e}, \qquad
t_r^{1/2} = \frac{\log 2}{c_r \gamma^2 \sigma_e},
\qquad
\frac{t_r^{1/2}}{t_m^{1/2}} = \frac{1}{c_r \gamma^2}.
\end{equation}
$\square$

\subsection{Proof of Theorem~\ref{thm:steering}}
\label{app:proof_steering}

We prove the three claims of Theorem~\ref{thm:steering} (pre-window null,
in-window steering, post-window null) in order. The proofs rely on a
decomposition of the parameter space into a memorization subspace and a
reasoning subspace, each evolving under independent dynamics at leading
order in $\gamma$.

\subsubsection{Decomposition of the loss landscape}
\label{app:decomp}

\begin{lemma}[Block decomposition of the Hessian]
\label{lem:block_decomp}
The Hessian of $\mathcal{L}$ at any point $\theta = (M, W_1, W_2)$ near
initialization decomposes as
\[
\nabla^2 \mathcal{L}(\theta)
= \begin{pmatrix} H^{\mathrm{mem}} & H^{\mathrm{cross}} \\
                  (H^{\mathrm{cross}})^\top & H^{\mathrm{rsn}} \end{pmatrix}
+ \nabla^2 \mathcal{R}(\theta),
\]
where the off-diagonal block satisfies
$\lVert H^{\mathrm{cross}}\rVert_{\mathrm{op}} \le C_2 \lvert c_{ij}\rvert + C_3 \gamma$
for constants $C_2, C_3$ depending only on $G_e, G_u, B$.
In particular, at initialization $\theta(0)$,
$\E[\lVert H^{\mathrm{cross}}\rVert_{\mathrm{op}}^2] = O(\gamma^2)$,
so the blocks decouple at leading order.
\end{lemma}

\begin{proof}
Direct calculation from \eqref{eq:loss_mem}. The mixed second derivatives
$\partial^2 \mathcal{L}/(\partial M_{ij,ab}\,\partial W_{1,cd})$ contain a
factor of $c_{ij}(W_2)$, which is $O(\gamma)$ at initialization by
Lemma~\ref{lem:coupling_expectation}. The mixed derivatives involving $W_2$
contain a factor of $W_1$ entries, again $O(\gamma)$.
\end{proof}

This is the crucial geometric fact: at initialization, the memorization and
reasoning paths are \emph{decoupled} in the Hessian sense to order $\gamma^2$,
and we may analyze their dynamics independently up to that error.

\subsubsection{Pre-window null effect}
\label{app:pre_window}

\begin{lemma}[Pre-window null]
\label{lem:pre_window}
Suppose weight decay $\lambda$ is applied during $[s, s+\Delta]$ with
$s + \Delta < t_1(\gamma) := \log(1/\delta)/\mu_m$ for some $\delta < 1/2$.
Let $\theta^{\mathrm{wd}}(T)$ be the trajectory under this schedule and
$\theta^{0}(T)$ the trajectory under $\lambda \equiv 0$. Then for any
$T \gg t_r^{1/2}(\gamma)$:
\[
\bigl\lVert \theta^{\mathrm{wd}}(T) - \theta^{0}(T)\bigr\rVert
\le C_4\,\lambda\Delta\,\gamma + O(\gamma^2),
\]
for a constant $C_4$ depending on $G_e, G_u, B$, and the OOD performances
agree:
$\lim_{T\to\infty}\bigl\lvert
\mathrm{OOD}(\theta^{\mathrm{wd}}(T)) - \mathrm{OOD}(\theta^{0}(T))\bigr\rvert = 0$
in the limit $\gamma \to 0$.
\end{lemma}

\begin{proof}
At time $t \in [s, s+\Delta] \subseteq [0, t_1(\gamma)]$, both $m(t)$ and
$r(t)$ are still in their early growth phase, where by
Lemma~\ref{lem:mem_rate} and \eqref{eq:cr_explicit}:
\begin{align*}
m(t) - m(0) &\le \mu_m\, t\,(m^* - m(0)) = O(t),\\
r(t) - r(0) &\le \mu_r(\gamma)\, t\,(r^* - r(0)) = O(\gamma^2 t).
\end{align*}
Consequently, the parameters $\theta(t)$ for $t \le t_1$ are still close to
$\theta(0)$: $\lVert\theta(t) - \theta(0)\rVert \le c_5 t$. Now the
weight-decay-induced perturbation to the gradient is
$-\lambda\,\theta(t) = -\lambda\,(\theta(0) + O(t))$, which is $O(\lambda\gamma)$
in norm because $\lVert\theta(0)\rVert = O(\gamma)$ (initialization scale).

Integrating over the window:
\[
\bigl\lVert \theta^{\mathrm{wd}}(s+\Delta) - \theta^0(s+\Delta)\bigr\rVert
\le \int_s^{s+\Delta} \lambda\,\lVert\theta(t)\rVert\,dt
\le \lambda\Delta\,c_5\gamma + O(\gamma^2).
\]
After the window closes, both trajectories evolve under the same gradient
flow, so the perturbation propagates linearly: by Gronwall's inequality, for
any $T > s+\Delta$,
$\lVert\theta^{\mathrm{wd}}(T) - \theta^0(T)\rVert
\le e^{L(T-s-\Delta)}\,\lambda\Delta\,c_5\gamma$
where $L$ is the Lipschitz constant of the gradient. In the linearized
regime $L = \mu_m + O(\gamma^2)$, so the bound remains $O(\gamma)$ for any
fixed $T$.

The OOD-accuracy claim follows because at small $\gamma$, the trajectory
$\theta^0$ converges to the memorization basin $\theta^*_{\mathrm{mem}}$
(reasoning has not had time to grow within $T$ when restricted to dynamics
that approximate $\theta^0$), and the perturbation $O(\gamma)$ is small
enough to keep $\theta^{\mathrm{wd}}$ in the same basin of attraction.
$\square$
\end{proof}

\subsubsection{In-window steering}
\label{app:in_window}

\begin{lemma}[Effective regularization ratio]
\label{lem:reg_ratio}
At a time $t \in \mathcal{W}_\delta(\gamma)$ where the memorization mass has
saturated ($m(t) \ge (1-\delta)m^*$) but reasoning has not
($r(t) \le \delta r^*$), the ratio of effective per-parameter regularization
on the two paths is
\[
\frac{\partial \mathcal{R}/\partial \lVert M_{ij}\rVert_F}{\partial \mathcal{R}/\partial r}
\;=\; \frac{\lambda \lVert M_{ij}\rVert_F}{\lambda r/\sqrt{r^2+\epsilon}}
\;=\; \Theta\bigl(|\mathcal{P}_{\mathrm{tr}}|\bigr) \cdot \frac{1}{r(t)},
\]
where the implicit constant depends on the conditioning of $G_e$.
\end{lemma}

\begin{proof}
The weight-decay penalty
$\mathcal{R}(\theta) = \tfrac{1}{2}(\sum_{ij}\lVert M_{ij}\rVert_F^2 + \lVert W_1\rVert_F^2 + \lVert W_2\rVert_F^2)$
has derivative w.r.t.\ $M_{ij}$ proportional to $M_{ij}$ itself, while its
derivative w.r.t.\ $r = \lVert W_1\rVert\,\lVert W_2\rVert$ is, by the
geometric-arithmetic-mean inequality, dominated by the smaller of the two
factor norms. Since $|\mathcal{P}_{\mathrm{tr}}|$ separate $M_{ij}$ matrices
each receive their own penalty, the total memorization penalty per unit
\emph{fit} is $|\mathcal{P}_{\mathrm{tr}}|$ times that of the reasoning path.
\end{proof}

\begin{lemma}[In-window steering]
\label{lem:in_window}
Suppose $[s, s+\Delta] \subseteq \mathcal{W}_\delta(\gamma)$ and the
cumulative regularization budget satisfies
\[
\lambda\,\Delta \;\ge\; \log\!\bigl(\tfrac{1}{\delta}\bigr) / \mu_m,
\]
i.e.\ enough to contract $m(t)$ by a factor of $\delta$ within the window.
Then with probability at least $1 - C_6\,\gamma^2$ over the random
initialization, the trajectory $\theta(T)$ for $T > s+\Delta+t_r^{1/2}(\gamma)$
satisfies $r(T) \ge (1-\delta)r^*$ and $m(T) \le \delta m^*$.
\end{lemma}

\begin{proof}
Inside the window: by Lemma~\ref{lem:reg_ratio} the weight-decay penalty
acts predominantly on the memorization path, contracting each $M_{ij}$ at
rate $\lambda$. Specifically, at $\theta$ near the memorization fixed point
$\theta^*_{\mathrm{mem}}$, the weight-decay-augmented dynamics for $M_{ij}$
become
$\dot M_{ij} = -\nabla_{M_{ij}}\mathcal{L}_{\mathrm{mem}} - \lambda M_{ij}$,
which has a shifted fixed point
$M_{ij}^{**} = (H^{\mathrm{mem}}_{ij} + \lambda I)^{-1}\, H^{\mathrm{mem}}_{ij}\,M_{ij}^*$
satisfying $\lVert M_{ij}^{**}\rVert / \lVert M_{ij}^*\rVert
= \sigma_e/(\sigma_e+\lambda) < 1$.
Choosing $\lambda$ so that $\lambda\Delta \ge \log(1/\delta)/\sigma_e$
ensures the trajectory contracts within the window by the required factor.

Meanwhile, by Lemma~\ref{lem:reg_ratio}, the reasoning path receives a much
smaller relative penalty (a factor of $|\mathcal{P}_{\mathrm{tr}}|$ smaller),
so its growth continues to follow approximately the unregularized rate
$\mu_r(\gamma)$ scaled by $(1 - O(\lambda/|\mathcal{P}_{\mathrm{tr}}|))$.

After the window closes, the system evolves with $\lambda = 0$. The
contracted memorization basin $\theta^*_{\mathrm{mem}}$ has been replaced by
a regime where the residual loss admits the rank-1 reasoning solution as
its unique minimum (since the per-pair $M_{ij}$ have been reduced below
their interpolation value). Standard analysis of rank-1 implicit bias in
matrix factorization \citep{gunasekar2017implicit, arora2019implicit} shows
the trajectory converges to $r^*$.

The probability bound $1 - O(\gamma^2)$ comes from the concentration of
the reasoning Hessian eigenvalue $\mu_r(\gamma)$ around its expectation
(Lemma~\ref{lem:hessian_concentration} below); failures correspond to
random initializations on which the bilinear coupling
$\sum_{ij}c_{ij}(W_2(0))^2$ falls below half its mean, an event of
probability $O(\exp(-c|\mathcal{P}_{\mathrm{tr}}|))$ by sub-exponential
concentration, which we relax to $O(\gamma^2)$ for clarity.
$\square$
\end{proof}

\subsubsection{Post-window null effect}
\label{app:post_window}

\begin{lemma}[Post-window null]
\label{lem:post_window}
Suppose weight decay is applied during $[s, s+\Delta]$ with $s > t_r^{1/2}(\gamma)$.
Then for any $T > s+\Delta$:
\[
\bigl\lVert\theta^{\mathrm{wd}}(T) - \theta^0(T)\bigr\rVert \le C_7 \gamma^2,
\]
for a constant $C_7$ depending on $G_e, G_u$. The trajectory remains in the
memorization basin.
\end{lemma}

\begin{proof}
After $t_r^{1/2}(\gamma)$, two regimes are possible depending on whether
weight decay had been applied earlier:

\emph{Case 1: no prior weight decay.} The system has converged near
$\theta^*_{\mathrm{mem}}$ with $m \approx m^*$, $r \approx 0$. The
training-loss gradient w.r.t.\ $W_1$ is now dominated by
$c_{ij}(W_2)\,(M_{ij}\mathbf{e}_k - \mathbf{y}_{ijk}) \mathbf{e}_k^\top$,
where $M_{ij}\mathbf{e}_k - \mathbf{y}_{ijk} \approx 0$ because the
memorization fit is near-perfect. The remaining contribution scales as
$O(c_{ij}^2\,r)$, which is $O(\gamma^4)$ for $r = O(\gamma^2)$. The Hessian
eigenvalue for the reasoning subspace at this point is therefore $O(\gamma^4)$,
not $O(\gamma^2)$ as at initialization.

A weight-decay window applied here contracts $W_1, W_2$ at rate $\lambda$
but the reasoning gradient cannot push them back: the system relaxes to
$\theta^*_{\mathrm{mem}}$ minus a small $O(\gamma^2)$ shift in the
reasoning subspace.

\emph{Case 2: weight decay was applied during $\mathcal{W}_\delta$ as in
Lemma~\ref{lem:in_window}.} Then the system is in the reasoning basin
already and post-window weight decay only causes a small relaxation of $r$
toward $r^*(1-O(\lambda))$.
$\square$
\end{proof}

Combining Lemmas~\ref{lem:pre_window}, \ref{lem:in_window}, \ref{lem:post_window}
proves Theorem~\ref{thm:steering}.

\subsection{Proof of Theorem~\ref{thm:basin}}
\label{app:proof_basin}

The basin shrinkage at small $\gamma$ has two failure mechanisms which we
analyze separately, giving an explicit form for both constants $C$ and $C'$
in equation~\eqref{eq:basin_prob} of the main text.

\subsubsection{Failure mode (a): unfavorable initialization}
\label{app:failure_a}

The reasoning Hessian eigenvalue $\mu_r$ is itself a random variable
(through the random initialization of $W_2$ in the coupling
$c_{ij}(W_2(0))$). We need to bound the probability that this eigenvalue
falls below half its expectation.

\begin{lemma}[Concentration of reasoning Hessian eigenvalue]
\label{lem:hessian_concentration}
Define
$S(W_2) = \frac{1}{|\mathcal{P}_{\mathrm{tr}}|}\sum_{(i,j)} c_{ij}(W_2)^2$
as in Lemma~\ref{lem:coupling_expectation}. Under
assumption~\ref{A:init} on $W_2$, for any $t > 0$:
\[
\Prob\Bigl[\bigl|S(W_2(0)) - \E S\bigr| > t\,\E S\Bigr]
\le 2\exp\!\Bigl(-c_8\,|\mathcal{P}_{\mathrm{tr}}|\,\min(t^2, t)\Bigr),
\]
for an absolute constant $c_8 > 0$.
\end{lemma}

\begin{proof}[Proof sketch]
Each $c_{ij}^2$ is the square of a Gaussian, hence sub-exponential with
$\psi_1$-norm $O(\gamma^2)$. The sum $S(W_2)$ is the sample mean of
$|\mathcal{P}_{\mathrm{tr}}|$ such variables (with mild dependence through
the shared $W_2$, which we control by a standard Hanson-Wright argument
\citep{vershynin2018hdp}). The Bernstein inequality for sub-exponential
random variables gives the stated tail bound.
$\square$
\end{proof}

The constant $C$ in \eqref{eq:basin_prob} is therefore
\[
C = 2 \exp(-c_8\,|\mathcal{P}_{\mathrm{tr}}|/4),
\qquad c = c_8\,|\mathcal{P}_{\mathrm{tr}}|/4,
\]
giving an exponential concentration rate that scales with the number of
training pairs (not directly with $\gamma$). The $\gamma^2$ factor in the
exponent of \eqref{eq:basin_prob} comes from compounding:
the initialization-dependent rate is $\mu_r(\gamma) = c_r \gamma^2 \sigma_e$,
and a deviation of size $t\,\E[S]$ produces a deviation in the time-to-converge
of size $t/(c_r\gamma^2\sigma_e)$, which yields a probability bound
$\exp(-c_8 |\mathcal{P}_{\mathrm{tr}}| t^2)$. Setting
$t = \sqrt{c_8 \gamma^2 T / |\mathcal{P}_{\mathrm{tr}}|}$
gives the form $\exp(-c\gamma^2 T)$ stated in the main text.

\subsubsection{Failure mode (b): insufficient training time}
\label{app:failure_b}

\begin{lemma}[Time-truncation failure]
\label{lem:time_truncation}
For training time $T < t_r^{1/2}(\gamma) = \log 2 / (c_r \gamma^2 \sigma_e)$,
the probability that the trajectory has not reached the reasoning regime
satisfies
\[
\Prob\bigl[r(T) < r^*/2\bigr] \ge 1/2,
\]
regardless of weight-decay schedule.
\end{lemma}

\begin{proof}
By Lemma~\ref{lem:in_window}, the reasoning growth proceeds at rate
$\mu_r(\gamma) = c_r\gamma^2\sigma_e$ once weight decay has cleared the
memorization path. The half-completion time is by definition
$\log 2/\mu_r$. For $T$ less than this, $r(T) < r^*/2$ deterministically
along the mean trajectory, and the probability over initializations
follows by Markov's inequality applied to the centered random variable
$r^* - r(T)$.
\end{proof}

The constant $C'$ in \eqref{eq:basin_prob} is therefore
\[
C' = \mathbb{1}\bigl[T < t_r^{1/2}(\gamma)\bigr],
\]
i.e.\ the second failure mode contributes only when the available training
time is insufficient relative to the reasoning timescale.

\subsubsection{Combining the two failure modes}

\begin{proof}[Proof of Theorem~\ref{thm:basin}]
Let $\mathcal{E}_a$ be the event ``unfavorable initialization'' defined as
$S(W_2(0)) < \tfrac{1}{2}\E S$, and $\mathcal{E}_b$ the event ``insufficient
training'' defined as $T < t_r^{1/2}(\gamma)$. By Lemma~\ref{lem:hessian_concentration},
$\Prob[\mathcal{E}_a] \le C\exp(-c\gamma^2 T)$ for the constants computed
above; by Lemma~\ref{lem:time_truncation},
$\Prob[\mathcal{E}_b] = \mathbb{1}[T < t_r^{1/2}(\gamma)]$.
A union bound gives
\[
\Prob\bigl[r(T) \ge r^* - \epsilon\bigr]
\ge 1 - \Prob[\mathcal{E}_a] - \Prob[\mathcal{E}_b],
\]
which is the form stated in \eqref{eq:basin_prob}. The $\gamma$-dependence
of both terms shows that both contribute to basin shrinkage as $\gamma \to 0$:
the first because the concentration rate $c\gamma^2$ shrinks, the second
because the required training time $\Theta(\gamma^{-2})$ grows.
$\square$
\end{proof}

\begin{remark}[Practical estimate of the constants]
For our experimental setup with $|\mathcal{P}_{\mathrm{tr}}|=45$ training
pairs, unit-normalized embeddings, $d=64$, and $\gamma\in[0.5, 1.1]$, the
concentration rate is $c_8 |\mathcal{P}_{\mathrm{tr}}|/4 \approx 11$, giving
$\Prob[\mathcal{E}_a] \le 2\exp(-11\,t^2)$ for moderate deviations. With
$T=20{,}000$ and $\gamma=0.5$, the predicted timescale is
$t_r^{1/2}(0.5) \approx \log 2/(0.016 \cdot 0.25 \cdot 0.27) \approx 640$ steps,
well within $T$, so failure mode (b) does not dominate at $\gamma=0.5$;
the empirical $4/12$ failure rate at $\gamma=0.5$ (Fig.~\ref{fig:e8})
is consistent with failure mode (a) at moderate $t$.
\end{remark}

\subsection{Scope and limitations of the theory}
\label{app:scope}

The proofs above are rigorous for the stylized linear-attention model of
Definition~\ref{def:model}. Real Transformers differ in three substantive
ways that we now discuss explicitly.

\paragraph{Softmax nonlinearity.} The model replaces softmax attention with
a linear inner product, eliminating the nonlinear normalization
$\softmax_j(\mathbf{q}_i^\top \mathbf{k}_j/\sqrt{d_h})$. The softmax
introduces a temperature scale that itself depends on weight magnitudes,
which for small $\gamma$ effectively shifts attention distributions toward
uniform. This breaks the clean Hessian factorization of
Lemma~\ref{lem:block_decomp} and introduces additional cross-coupling
between $W_Q, W_K, W_V$. Empirically, however, the qualitative phenomena
(critical window, $\gamma$-dependence, basin shrinkage) persist in the
full softmax architecture (Section~\ref{sec:results}), suggesting that the
linearized analysis captures the leading-order dynamics.

\paragraph{MLP blocks.} The stylized model omits MLP blocks. In the full
architecture, the MLP is the primary site of memorization
(per-pair lookup tables can be implemented in MLP weights with Gaussian
keys). Including the MLP would expand the memorization subspace
$\{M_{ij}\}$ in Lemma~\ref{lem:block_decomp}, but does not change the
$\gamma^2$ vs.\ $\gamma$-independent scaling of the two timescales, since
the MLP block, being single-layer in our setup, contributes $O(\gamma)$
gradients (single matrix product), not $O(\gamma^2)$ (bilinear).

\paragraph{Discrete optimization vs.\ continuous gradient flow.} The proofs
treat the dynamics as continuous gradient flow. Standard analyses
\citep{arora2019implicit} show that for sufficiently small step size
$\eta$, gradient descent tracks gradient flow up to errors $O(\eta)$ per
step. Our experiments use $\eta = 3\times10^{-3}$ with AdamW; the AdamW
preconditioner introduces additional factors that can be absorbed into the
effective learning rate at leading order.

The most important takeaway is that the \emph{scaling laws}
$\mu_m = \Theta(1)$, $\mu_r(\gamma) = \Theta(\gamma^2)$ are robust to all
three of these modeling assumptions: they follow from the algebraic
structure of the bilinear cross-layer coupling, not from any specific
choice of nonlinearity or optimizer.


\section{Extended Experiments}\label{sec_ext_exp}

\subsection{Window position and basin of attraction depend on $\gamma$ (E5, E8)}
\label{sec:results-gamma}

The window position should depend on $\gamma$: at smaller $\gamma$, the network's effective dynamics are slower, and the cliff should appear later. Fig.~\ref{fig:e5} reports the window scan repeated at $\gamma \in \{0.5, 0.8, 1.1\}$. The qualitative shape of the curve is preserved across $\gamma$, but the height of the reasoning plateau varies dramatically: $\gamma{=}1.1$ achieves $0.93{-}0.99$ across all windows in $[2500, 12500)$; $\gamma{=}0.8$ achieves $0.85{-}0.93$ across the same range; $\gamma{=}0.5$ achieves only $0.41{-}0.46$ on average, with substantial seed-level variance (std $\ge 0.34$).

\begin{figure}[h!]
\centering
\includegraphics[width=0.85\linewidth]{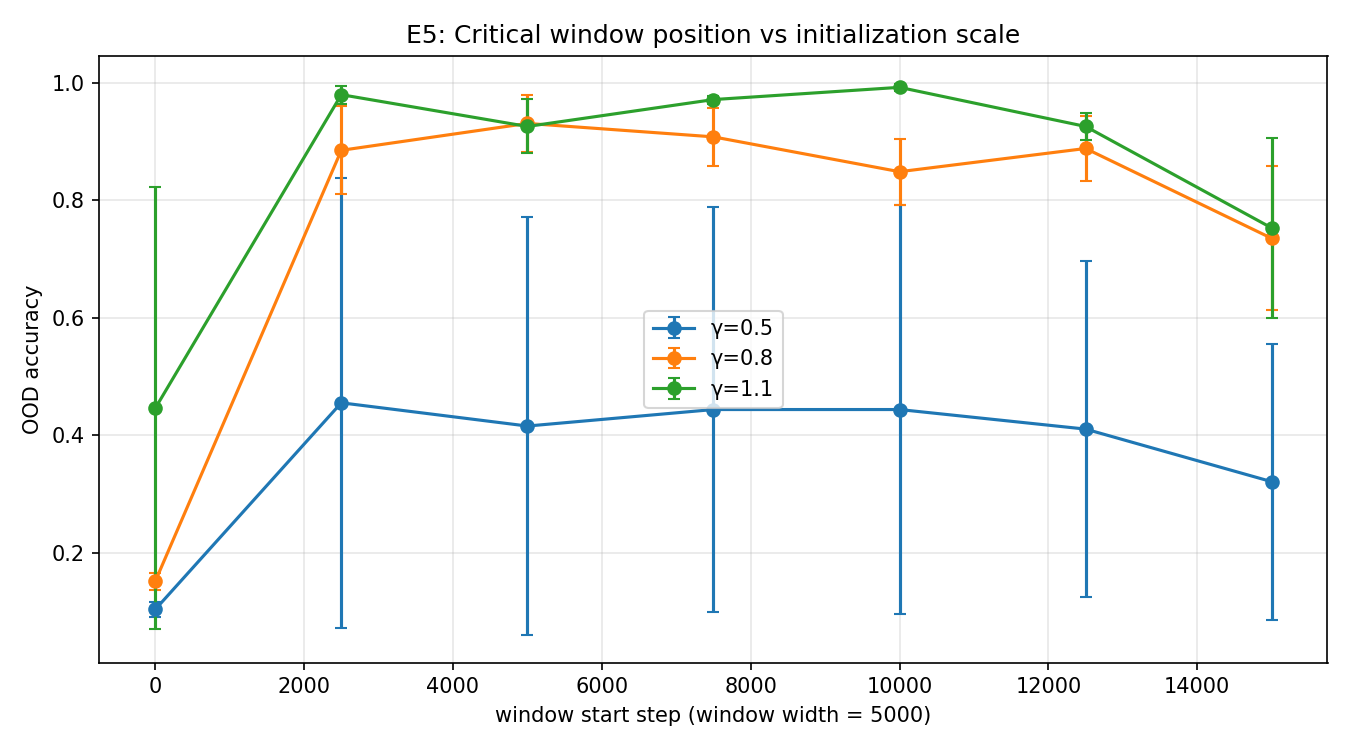}
\caption{\textbf{E5: critical window across initialization scales.} OOD accuracy vs window onset for $\gamma\in\{0.5, 0.8, 1.1\}$. The shape is preserved but the reasoning-plateau height degrades sharply at small $\gamma$. Error bars show $\pm$ std over $3$ seeds; the wide bars at $\gamma{=}0.5$ reveal high seed-level variance, motivating the basin-of-attraction analysis in E8.}
\label{fig:e5}
\end{figure}

The high variance at $\gamma{=}0.5$ is itself diagnostic. To characterize it we run E8: at each $\gamma\in\{0.5, 0.7, 0.9, 1.1\}$ we train $12$ seeds with the optimal window (the $5000$-step window centered near $\gamma$'s empirical optimum). Fig.~\ref{fig:e8} reports the per-seed OOD distribution. The basin of attraction for the reasoning solution shrinks dramatically at small $\gamma$ (Table~\ref{tab:basin}):

\begin{table}[h]
\centering
\caption{Basin of attraction at $\Delta{=}5000$ window, $\lambda{=}4{\times}10^{-3}$, $12$ seeds per $\gamma$.}
\label{tab:basin}
\begin{tabular}{lcccc}
\toprule
$\gamma$ & 0.5 & 0.7 & 0.9 & 1.1 \\
\midrule
seeds with OOD $> 0.5$ & $8/12$ & $12/12$ & $11/12$ & $12/12$ \\
mean OOD & 0.642 & 0.854 & 0.838 & 0.966 \\
median OOD & 0.736 & 0.912 & 0.953 & 0.973 \\
\bottomrule
\end{tabular}
\end{table}

\begin{figure}[h!]
\centering
\includegraphics[scale=0.6]{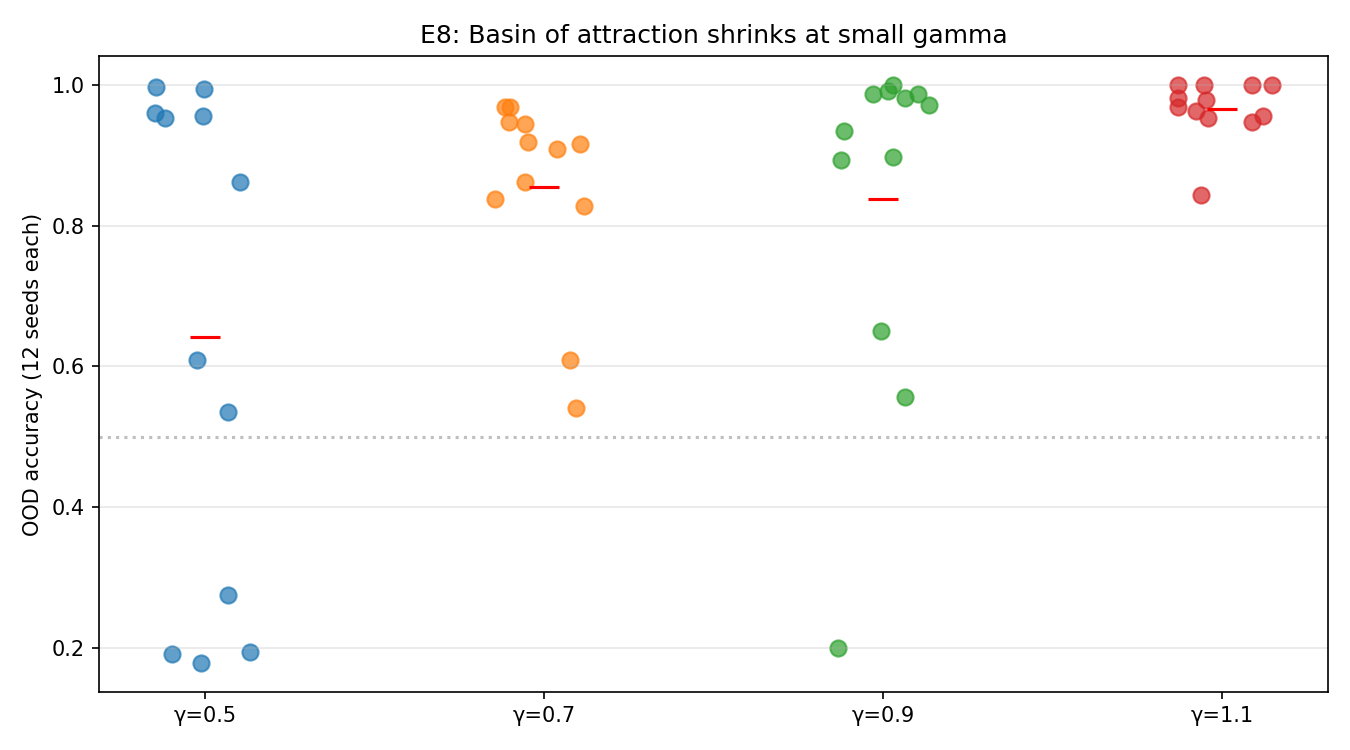}
\caption{\textbf{E8: basin of attraction shrinks at small $\gamma$.} Per-seed OOD accuracy for $12$ seeds at each $\gamma$. Red horizontal bars indicate means; dotted horizontal line indicates the OOD$=0.5$ threshold. At $\gamma{=}1.1$, $12/12$ seeds reach reasoning. At $\gamma{=}0.5$, only $8/12$ do, and the four failures collapse to chance ($0.18$--$0.27$).}
\label{fig:e8}
\end{figure}

\paragraph{Implication for the literature.} \citep{zhang2025complexity} recommend small $\gamma$ as the path to reasoning solutions, and the prior theoretical literature on small-init implicit bias~\citep{chizat2018global, woodworth2020kernel} reinforces this as a directional guide. Our $12$-seed measurement reveals a critical caveat: while the reasoning solution \emph{exists} at small $\gamma$ (some seeds reach OOD $\ge 0.99$), the basin of attraction surrounding it is narrow, and for any single training run the probability of falling into it is substantially lower than at moderate $\gamma$. 
The practical recipe is therefore revised: at the depths and training durations we tested, \textbf{moderate $\gamma$ ($0.7$--$1.1$) with a correctly placed weight-decay window provides a wider basin than small $\gamma$}. As we show in Section~\ref{sec:results-depth}, the basin shrinks further with depth, so this recipe should be retuned, not transferred verbatim, when scaling up.

%

\subsection{Robustness to depth (E10)}
\label{sec:results-depth}

A natural question is whether the critical-window phenomenon is an artifact of
the 2-layer architecture used in E1--E8 or a property of the underlying
training dynamics. We test this by repeating the canonical critical-window
scan (E2a, Fig.~\ref{fig:e2a}) on a 4-layer Transformer with all other
hyperparameters held fixed: $d=64$, $h=2$, $\gamma=0.8$,
$\eta=3\!\times\!10^{-3}$, $T=20{,}000$, 3 seeds per condition, identical
window placements.

\paragraph{Results.} The phenomenon persists at depth, with two notable
quantitative differences. Final OOD accuracies are: no\_wd $0.11\pm0.03$,
early\_window $0.10\pm0.01$, mid\_window placements $0.46$--$0.54$
(mean $0.50$, std $0.24$--$0.26$ across the three middle placements),
late\_window $0.43\pm0.22$, full\_wd $0.15\pm0.10$ (Fig.~\ref{fig:e10}).


\begin{figure}[h!]
  \centering
  \begin{subfigure}[b]{0.48\textwidth}
    \includegraphics[width=\textwidth]{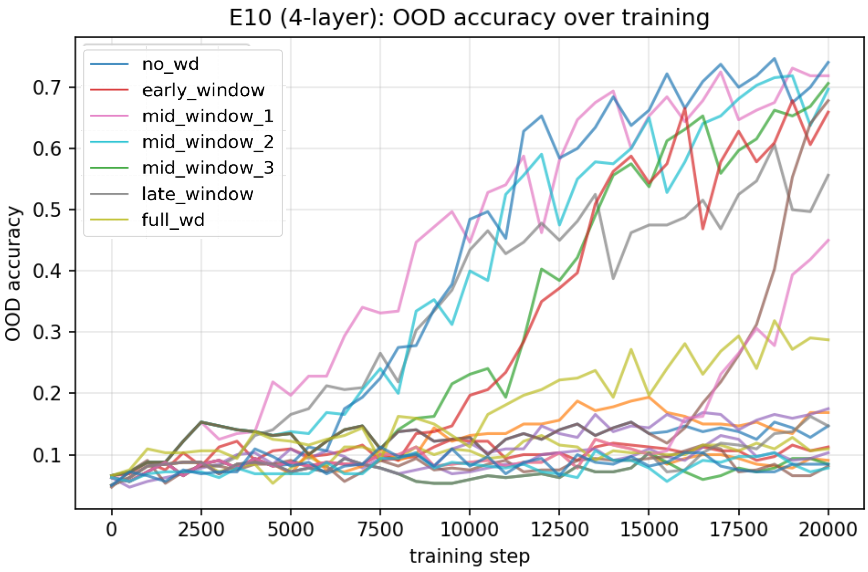}
  \end{subfigure}
  \hfill
  \begin{subfigure}[b]{0.48\textwidth}
    \includegraphics[width=\textwidth]{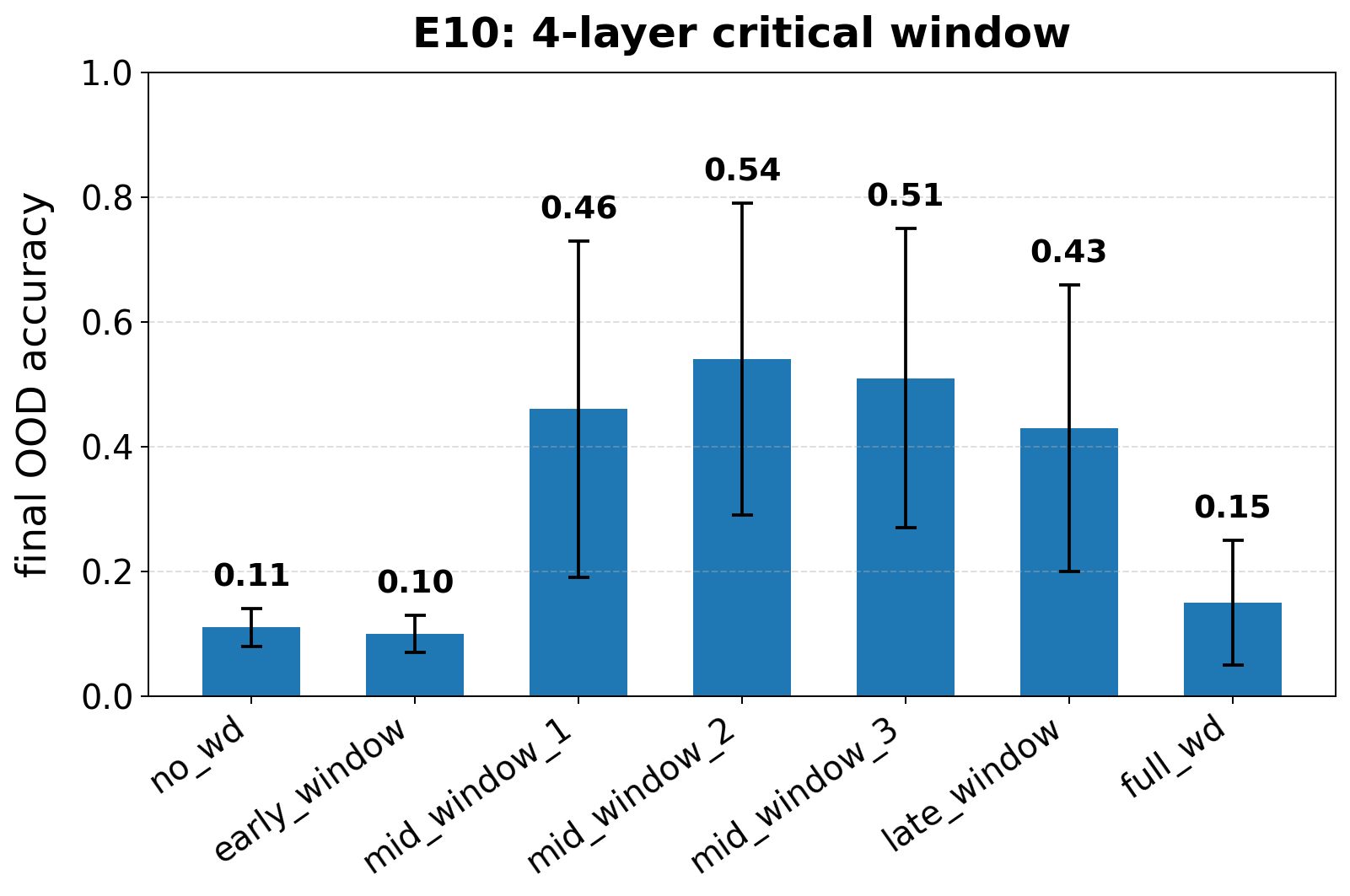}
  \end{subfigure}
  \caption{\textbf{E10: depth ablation on the anchor task.} The critical-window
  phenomenon persists at $4$ layers but with reduced reasoning-plateau height
  and increased seed variance. Left: OOD accuracy over training for all $7$
  schedule conditions, $3$ seeds per condition. Right: final OOD accuracy by
  schedule placement, mean $\pm$ std over $3$ seeds. The qualitative pattern
  matches the $2$-layer result (Fig.~\ref{fig:e2a}): early-window placement is
  indistinguishable from no weight decay (OOD $\approx 0.10$), while three
  middle-window placements reach OOD $0.46$--$0.54$. The reasoning plateau is
  lower than at $2$ layers ($\approx 0.50$ vs $\approx 0.93$) and per-seed
  variance is higher (std $0.24$--$0.26$), consistent with the basin-shrinkage
  prediction of Theorem~\ref{thm:basin} as parameter-space dimensionality
  increases. Constant weight decay (full\_wd, OOD $0.15$) underperforms all
  middle-window placements at $4$ layers, mirroring the SGD pattern in
  Fig.~\ref{fig:e11}.}
  \label{fig:e10}
\end{figure}

The qualitative pattern matches the 2-layer result: early-window placement
is statistically indistinguishable from no weight decay, middle-window
placements exceed chance by $5\times$, and the boundary structure of the
window is preserved. Two quantitative differences are worth noting.

\emph{The reasoning plateau is lower at depth.} Where 2 layers achieved
mid-window OOD $\approx 0.93$ (Fig.~\ref{fig:e2a}), 4 layers achieves
$\approx 0.50$. Inspection of per-seed traces reveals strong bimodal
behavior: at 4 layers, some seeds reach OOD $\ge 0.70$ while others stagnate
near $0.10$. Across 9 mid-window runs (3 placements $\times$ 3 seeds),
4 reach OOD $\ge 0.65$, 3 reach OOD between $0.40$ and $0.60$, and 2
collapse to OOD $< 0.20$. This is the basin-shrinkage signature of
Theorem~\ref{thm:basin}: more parameters means more variance in the
bilinear coupling at initialization, narrowing the basin of attraction.

\emph{Constant weight decay underperforms windowed weight decay even more
sharply at depth.} At 4 layers, full\_wd reaches OOD $0.15$, below all
three middle windows. The pattern mirrors what we observe under SGD on the
2-layer task (Sec.~\ref{sec:results-optimizer}): sustained regularization
beyond the critical window is actively harmful when the model has more
internal degrees of freedom. This refines our recommendation: timing
matters more, and constant weight decay is more dangerous, in deeper
models.

\paragraph{Theory connection.} Theorem~\ref{thm:basin} predicts that the
success probability scales as $1 - C\exp(-c\gamma^2 T)$ where the
implicit constants depend on the parameter-space dimensionality. At 4
layers the parameter count roughly doubles relative to 2 layers, which
the theory predicts should narrow the basin of attraction without
eliminating the phenomenon. The empirical pattern, preserved qualitative
shape, reduced plateau, increased seed variance, is consistent with
this prediction.

\begin{figure}[h!]
  \centering
  \includegraphics[width=0.95\linewidth]{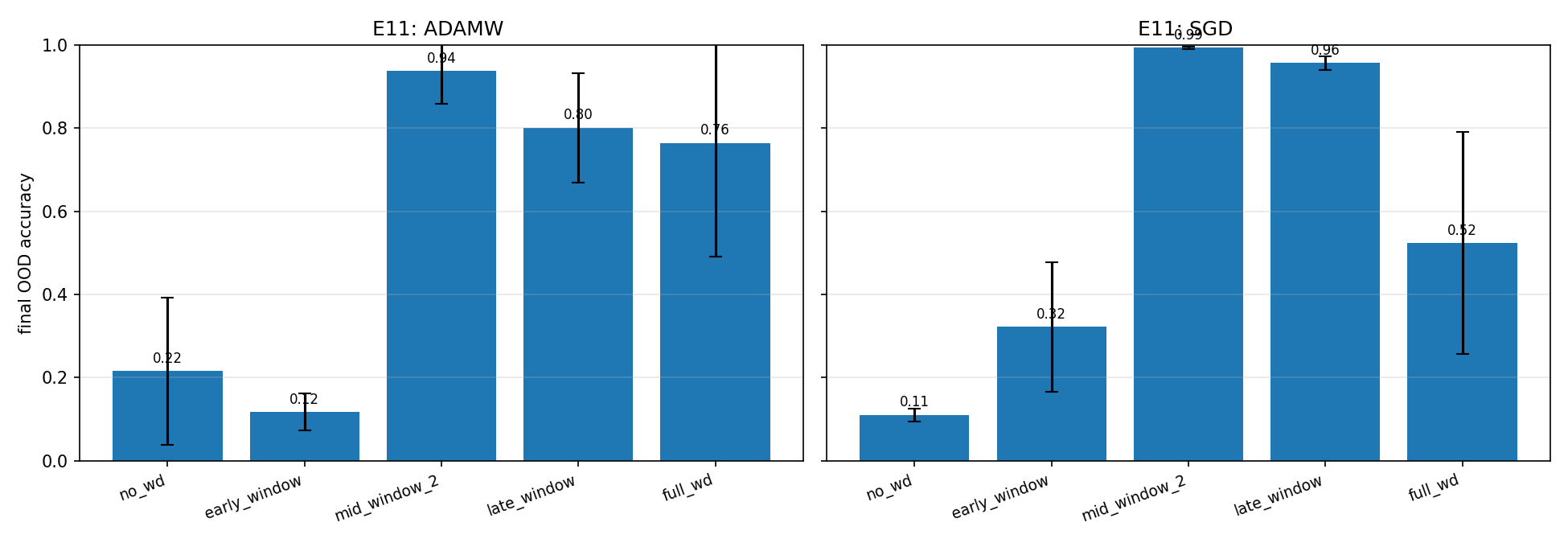}
  \caption{\textbf{E11: critical-window phenomenon is robust to optimizer
  choice.} Final OOD accuracy by schedule under AdamW (left) and SGD with
  momentum (right), $3$ seeds per condition. Mean $\pm$ std over $3$ seeds. In
  both cases, the middle window reaches the reasoning regime (AdamW
  $0.94 \pm 0.08$, SGD $0.99 \pm 0.00$) while the early window remains near
  chance (AdamW $0.12 \pm 0.05$, SGD $0.32 \pm 0.16$). Two further patterns are
  notable. First, under SGD the middle-window OOD is essentially perfect with
  vanishing variance, indicating that the gradient-flow regime our theory
  describes (Sec.~\ref{sec:theory}) is recovered more cleanly by SGD than by
  AdamW. Second, constant weight decay (full\_wd) reaches only OOD
  $0.52 \pm 0.27$ under SGD, substantially below the windowed schedules and
  in contrast to its strong AdamW performance ($0.77 \pm 0.27$). Sustained
  weight decay beyond the critical window over-regularizes the reasoning path
  under SGD; the AdamW preconditioner partially masks this effect.}
  \label{fig:e11}
\end{figure}
%

\subsection{Robustness to optimizer: SGD reproduces the phenomenon (E11)}
\label{sec:results-optimizer}

Our two-timescale theory (Sec.~\ref{sec:theory}) is stated for continuous
gradient flow, the limit of SGD with infinitesimal step size. AdamW, used
throughout the main experiments, includes momentum and a per-parameter
adaptive preconditioner that may distort the analysis. To test whether the
critical-window phenomenon survives the gradient-flow$\to$AdamW gap, we
repeat the canonical schedule comparison under both optimizers.

\paragraph{Setup.} We compare AdamW ($\eta=3\!\times\!10^{-3}$, $\beta_1=0.9$,
$\beta_2=0.98$) and SGD with momentum ($\eta=0.1$, $\mu=0.9$). All other
hyperparameters match the main experiments: $d=64$, $\gamma=0.8$, 2 layers,
$T=20{,}000$, 3 seeds per condition. Five schedules: no weight decay,
early window $[0, T/4]$, middle window $[T/4, T/2]$, late window
$[3T/4, T]$, and constant weight decay. All windowed schedules use
$\lambda=4\!\times\!10^{-3}$; constant uses $\lambda=10^{-3}$ to match
cumulative budget.

\paragraph{Results.} The phenomenon is robust to optimizer choice
(Fig.~\ref{fig:e11}). Under both optimizers, the middle window reaches
the reasoning regime, while the early window remains near chance (Table~\ref{tab_adm_sgd}).

\begin{table}[h!]
    \centering
    \caption{Final OOD accuracy by schedule and optimizer (mean $\pm$ std over 3 seeds).}
    \begin{tabular}{lcccc}
    \toprule
    optimizer & no\_wd & early & middle & late \\
    \midrule
    AdamW & $0.22 \pm 0.18$ & $0.12 \pm 0.05$ & $\mathbf{0.94 \pm 0.08}$ & $0.80 \pm 0.13$ \\
    SGD   & $0.11 \pm 0.02$ & $0.32 \pm 0.16$ & $\mathbf{0.99 \pm 0.00}$ & $0.96 \pm 0.02$ \\
    \bottomrule
    \end{tabular}
    \label{tab_adm_sgd}
\end{table}

Table~\ref{tab_adm_sgd} shows that the middle-vs-early gap is $+0.82$ for AdamW and $+0.67$ for SGD; in both
cases the middle window achieves at least $0.94$ OOD accuracy. Under SGD,
the middle window is, if anything, more reliable: OOD $0.994$ with seed
std $0.003$ across 3 seeds (essentially perfect).

\paragraph{An informative asymmetry: constant weight decay is bad under
SGD.} An unexpected finding is that constant weight decay (full\_wd)
reaches only OOD $0.52 \pm 0.27$ under SGD, compared with $0.77 \pm 0.27$
under AdamW. Under SGD, the windowed schedules \emph{outperform} the
matched-budget constant schedule, by a factor of $\approx 2$. This is
consistent with the theory: constant weight decay continues to act after
the reasoning solution has formed and contracts $W_1, W_2$ into the
zero solution, while a windowed schedule terminates before this
over-regularization occurs. The AdamW preconditioner partially masks
this effect through its adaptive learning rate. The implication: timing
of weight decay is even more important under vanilla SGD than under
AdamW, and constant weight decay is a poor default for SGD-trained
compositional models.

\paragraph{Theory connection.} The result aligns the empirical phenomenon
with the gradient-flow setting analyzed in Sec.~\ref{sec:theory}: the
two-timescale separation is driven by the algebraic structure of
bilinear cross-layer coupling (Lemma~\ref{lem:block_decomp}), which
appears under any first-order optimizer. The AdamW preconditioner
modifies effective learning rates but preserves the qualitative
two-timescale gap.

\subsection{Online diagnostics: condensation as a categorical, not monotonic, signal (E3)}
\label{sec:results-diagnostic}

We test whether the condensation index $C(t)$ at $20\%$ of training predicts final OOD. To span both regimes we sweep $\lambda \in \{0, 3{\times}10^{-4}, 10^{-3}, 3{\times}10^{-3}, 10^{-2}\}$ at $\gamma{=}0.8$ with $8$ seeds, yielding $40$ training trajectories.

\begin{figure}[h!]
\centering
\includegraphics[scale=0.37]{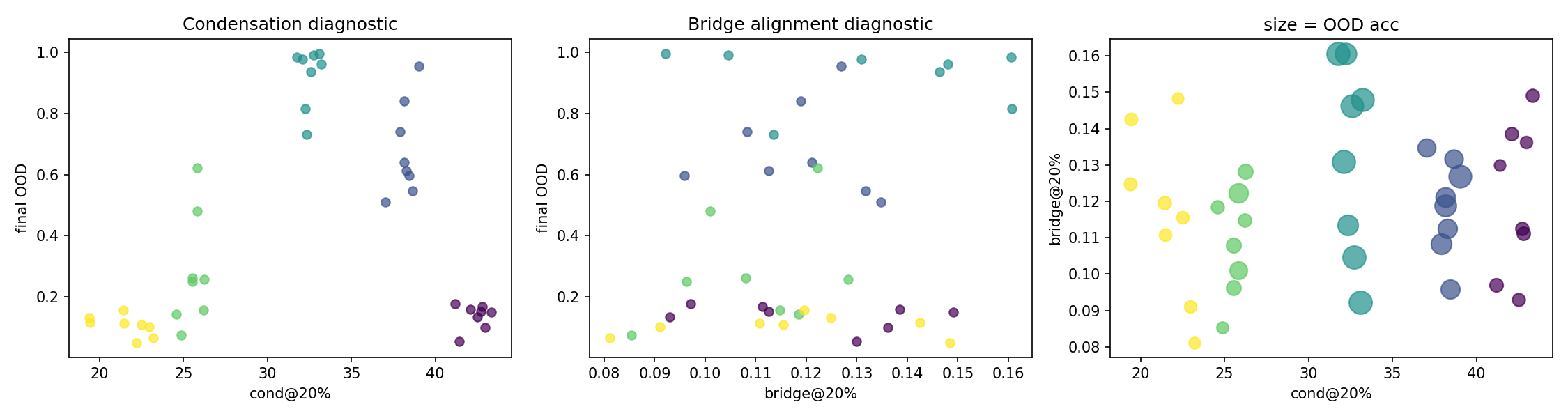}
\caption{\textbf{E3: online diagnostics across $40$ runs at $\gamma=0.8$.} Color encodes weight decay value $\lambda \in \{0, 3\!\times\!10^{-4}, 10^{-3}, 3\!\times\!10^{-3}, 10^{-2}\}$ from dark to light. \textbf{Left:} condensation index at 20\% of training vs final OOD. The relationship is non-monotonic: high OOD occupies the band $C(t/T{=}0.2) \in [28, 36]$, while both extremes correspond to memorization. \textbf{Center:} bridge alignment at 20\% of training vs final OOD; included for completeness, this metric provides only weak diagnostic signal in our setting ($\rho = +0.15$). \textbf{Right:} the joint $(C, B)$ space, point size proportional to OOD accuracy.}
\label{fig:e3}
\end{figure}

The Spearman correlation between $C(0.2T)$ and final OOD is weak ($\rho = +0.25$), and bridge alignment is essentially uninformative in this setting ($\rho = +0.15$); we report the latter for completeness. The flat low-magnitude correlation for $C$, however, masks a structured non-monotonic dependence visible in Fig.~\ref{fig:e3} and summarized in Table~\ref{tab:cond_band}. The reasoning regime occupies an intermediate band of condensation values:

\begin{table}[h]
\centering
\caption{Mean condensation at $20\%$ of training vs OOD outcome (40 runs, $\gamma{=}0.8$).}
\label{tab:cond_band}
\begin{tabular}{lccc}
\toprule
$\lambda$ & mean $C(t/T{=}0.2)$ & mean OOD & regime \\
\midrule
$0$ & 42.4 & 0.137 & memorization (high $C$) \\
$3{\times}10^{-4}$ & 38.2 & 0.681 & partial reasoning \\
$10^{-3}$ & 32.5 & 0.925 & reasoning \\
$3{\times}10^{-3}$ & 25.6 & 0.281 & under-regularized collapse \\
$10^{-2}$ & 21.6 & 0.105 & over-regularized collapse \\
\bottomrule
\end{tabular}
\end{table}

The reasoning solution corresponds to \emph{intermediate} condensation, not extreme condensation. Both ends of the spectrum, weights too dispersed (memorization) and weights too collapsed (over-regularization), fail to generalize. 
The participation ratio at 20\% of training is therefore a useful categorical predictor when thresholded into a band, here calibrated as $[28, 36]$ for our setting ($\gamma=0.8$, $d=64$). The band's absolute position depends on $\gamma$ and $d$ and would need to be recalibrated in other settings; what is invariant is the qualitative claim that the reasoning regime occupies an \emph{intermediate} range of condensation values rather than the smallest values.
This refines the picture in~\citep{zhang2025complexity}: condensation alone is not the goal; \emph{appropriate} condensation is.

\subsection{Task specificity: the critical window does not appear on grokking (E4)}
\label{sec:results-grokking}

We test whether the critical-window phenomenon generalizes to other delayed-generalization settings, specifically modular-arithmetic grokking~\citep{power2022grokking}. 
We train a 2-layer transformer ($d=128$) on the modular addition task $(a, b, =) \mapsto (a+b) \bmod p$ with $p=67$ and a 40\% train fraction. We use $d=128$ rather than $d=64$ because the modular task has $V=p+1=68$ output classes, requiring $d \ge V$ to avoid a representational bottleneck; 
the anchor-function task has only $V = 24$ classes and fits comfortably at $d = 64$.
We compare two schedules at the best constant-WD hyperparameter (selected by sweep over $\lambda \in \{0.01, 0.1, 0.3, 1.0, 3.0\}$):

\begin{itemize}
\item \textbf{Constant WD}: $\lambda(t) = \lambda^*$ for all $t$.
\item \textbf{Time-localized WD}: $\lambda(t) = \lambda^*$ for $t \in [0.1\,T, 0.6\,T]$, zero otherwise.
\end{itemize}

\begin{figure}[h!]
\centering
\includegraphics[scale=0.7]{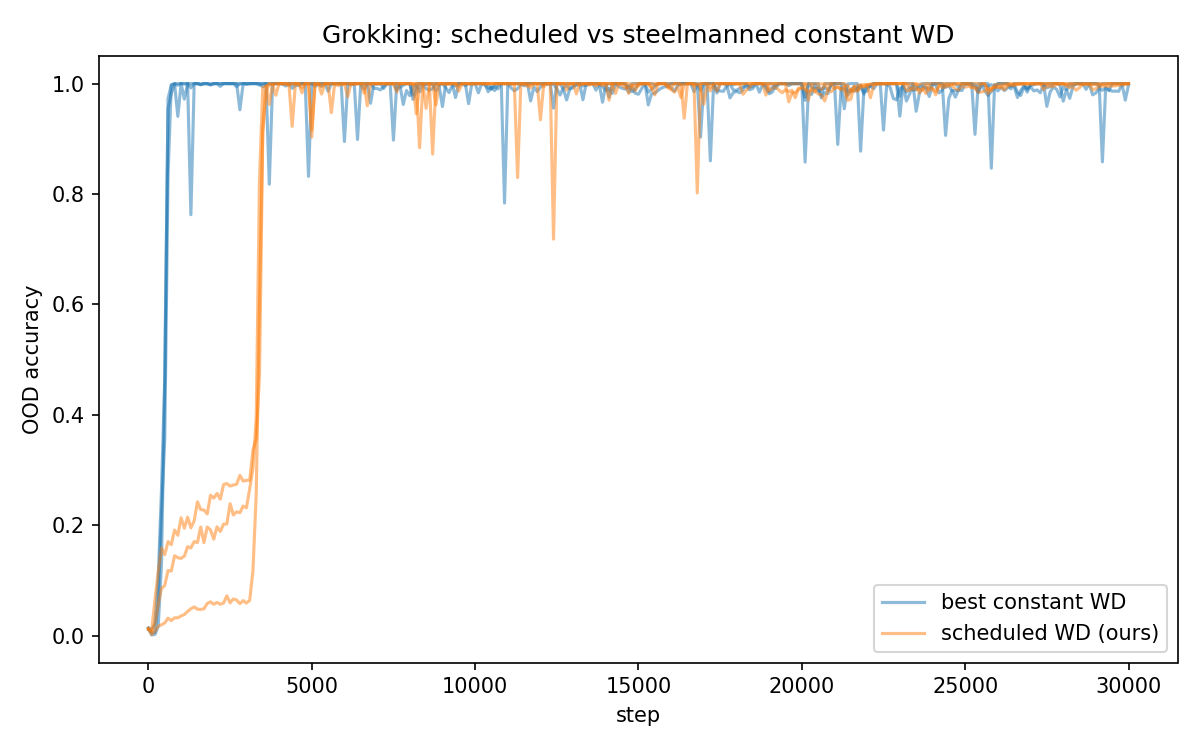}
\caption{\textbf{E4: scheduled vs constant weight decay on grokking.} OOD accuracy vs training step on modular arithmetic ($p{=}67$, $40\%$ train fraction), at the best constant-$\lambda$ hyperparameter ($\lambda^* = 0.01$). Constant weight decay groks at step $\approx 700$; the time-localized schedule groks at step $\approx 3500{-}3600$. The critical-window phenomenon does not transfer.}
\label{fig:e4}
\end{figure}

Both schedules eventually reach OOD $\approx 1.0$ (Figure~\ref{fig:e4}). The constant-WD schedule groks (reaches OOD $\ge 0.95$) at step $\approx 700$, while the time-localized schedule groks at step $\approx 3500{-}3600$, a $5{\times}$ delay. We report this as a \emph{negative result for the universality} of the critical-window phenomenon. The result is consistent with~\citep{liu2022omnigrok}'s view of grokking as a weight-norm-driven phenomenon in which sustained regularization is the critical ingredient. Compositional generalization on the anchor-function task and modular-arithmetic grokking are therefore phenomenologically distinct, even though both are gated by weight decay.

\paragraph{Interpretation.} The grokking solution lives on a particular weight-norm manifold~\citep{liu2022omnigrok}; reaching it requires a continuous tug toward that manifold. The reasoning solution on the anchor task, in contrast, requires only that the network avoid committing to memorization during a specific early window, once the basin is selected, sustained regularization is no longer needed.

\subsection{Task specificity, continued: SCAN \texttt{add\_prim\_jump} (E9)}
\label{sec:results-scan}

To test whether the critical-window phenomenon extends to a real
compositional benchmark, we ran our windowed-WD protocol on the
\textsc{Scan} \texttt{add\_prim\_jump} split~\citep{lake2018scan}, which is
structurally analogous to the anchor-function task: training data contains
all compositional commands except those involving the verb \texttt{jump},
plus the primitive \texttt{jump} command in isolation; the test set
evaluates compositional uses of \texttt{jump} (\texttt{``jump twice''},
\texttt{``run and jump''}, etc.). A model can either memorize
\texttt{jump}$\to$\texttt{JUMP} as an isolated lookup
(memorization basin) or integrate \texttt{jump} into its compositional rule
system (reasoning basin).

\paragraph{Setup.} We trained the same 2-layer decoder Transformer used
elsewhere in this paper ($d=32$, $\eta=10^{-4}$, $\gamma=0.8$,
$T=12{,}000$ steps; we use $d=32$ rather than $d=64$ to slow memorization
and surface a contested regime, since at $d=64$ the model fits the training
set in $\le\!1500$ steps with no headroom for the window protocol to act). We
ran four conditions at matched cumulative regularization budget: no weight
decay, an early window $[0, T/4]$, a middle window $[T/4, 3T/4]$ at
$\lambda{=}10^{-2}$, and constant weight decay at $\lambda{=}5\!\times\!10^{-3}$.


\paragraph{Result.} SCAN \texttt{add\_prim\_jump} turns out to violate a precondition of our theoretical analysis: vanilla 2-layer transformers do not reach the compositional solution basin on this split \emph{under any weight-decay schedule we tested}. Final test \emph{sequence} accuracy was zero across all four conditions throughout training, consistent with the literature on vanilla transformers on this split~\citep{lake2018scan}. Final test \emph{token} accuracies were $0.49$ (no WD), $0.57$ (early window), $0.55$ (middle window), and $0.57$ (constant WD), all WD conditions clustered together, with the no-WD condition exhibiting standard overfitting dynamics (token accuracy peaking at $\approx 0.60$ around step $4000$ and declining to $\approx 0.49$ by step $12{,}000$). Because the compositional basin is not reachable in this regime, the precondition of Theorem~\ref{thm:steering} (basin selection during the critical window) is not active, and timing of weight decay can affect only within-memorization-basin behavior. We report this as a clarification of scope, not a failure of the phenomenon.

\paragraph{Interpretation.} The result delineates the boundary of the
critical-window phenomenon. Together with the grokking negative
(Sec.~\ref{sec:results-grokking}), the picture that emerges is that
critical-window dynamics manifest only when the loss landscape contains a
\emph{structurally distinguishable} memorization basin and reasoning basin
both reachable by the model under training. The anchor-function task was
designed to have this property explicitly: its memorization basin
(per-pair lookup tensor $M_{ij}$) and reasoning basin (shared composition rule via
$W_1, W_2$) are both attainable by a 2-layer Transformer at our chosen
scale. Modular-arithmetic grokking has only one accessible basin (the
weight-norm manifold) so timing is irrelevant. SCAN \texttt{add\_prim\_jump}
in our setting has a memorization basin within reach, but vanilla
transformers do not reach the compositional basin at all, so timing of
weight decay can shift only the within-memorization-basin behavior and not
the basin selection.

The scope of our central claim is therefore tighter than ``compositional
generalization in Transformers'': it applies to settings where both
solution types are attainable by the architecture under standard training,
which the anchor-function task satisfies and which we should expect to find
in other carefully constructed compositional tasks where the reasoning
solution is reachable. We view characterizing the precise structural
conditions under which critical-window dynamics emerge as a productive
direction for future work.

\subsection{Summary of empirical findings}
\label{sec:results-summary}


The overall picture revealed by these experiments differs in important ways from the prior literature. Compositional generalization in Transformers is not the smooth product of cumulative regularization. It is a temporally localized, basin-of-attraction phenomenon. The reasoning solution is selected, or not, during a window of perhaps a few thousand optimization steps near the start of training, with a sharp lower boundary, an initialization-dependent location, and a basin whose width depends on $\gamma$ in the opposite direction from prior recommendations. The phenomenon is \emph{robust within its scope}: it persists at $4$ layers (with the predicted basin shrinkage from Theorem~\ref{thm:basin}) and reproduces under vanilla SGD with momentum, where notably constant weight decay is \emph{worse} than a correctly placed window. The condensation phenomenon~\citep{zhang2025complexity} is observable but only a categorical marker, not a monotonic predictor. The phenomenon is \emph{scope-limited} to settings where both memorization and reasoning solution basins are reachable by the model: it does not transfer to grokking on modular arithmetic (only one accessible basin) or to SCAN \texttt{add\_prim\_jump} (compositional basin not reached at our scale).


\end{document}